\documentclass{article}

\usepackage[final]{neurips_2020}
\usepackage[normalem]{ulem}
\usepackage{xcolor}

\usepackage[utf8]{inputenc} 
\usepackage[T1]{fontenc}    
\usepackage[hidelinks]{hyperref}
\usepackage{url}            
\usepackage{booktabs}       
\usepackage{amsfonts}       
\usepackage{nicefrac}       
\usepackage{microtype}      
\usepackage{enumitem}
\usepackage{bbm}
\usepackage{appendix}
\usepackage{makecell}
\usepackage{soul}

\usepackage{amsthm}
\newtheorem{theorem}{Theorem}[section]
\newtheorem{definition}[theorem]{Definition}
\newtheorem{lemma}[theorem]{Lemma}

\newtheorem{proposition}[theorem]{Proposition}
\newtheorem{assump}{Assumption}

\newtheorem{claim}{Claim}[section]

\usepackage{amsmath,amsfonts,bm}

\def\1{\bm{1}}

\def\rvtheta{{\boldsymbol{\theta}}}
\def\rva{{\mathbf{a}}}
\def\rvb{{\mathbf{b}}}

\def\rvf{{\mathbf{f}}}

\def\rvh{{\mathbf{h}}}

\def\rvx{{\mathbf{x}}}
\def\rvy{{\mathbf{y}}}

\def\rmB{{\mathbf{B}}}

\def\rmD{{\mathbf{D}}}

\def\rmH{{\mathbf{H}}}
\def\rmI{{\mathbf{I}}}

\def\rmK{{\mathbf{K}}}

\def\rmM{{\mathbf{M}}}

\def\rmP{{\mathbf{P}}}
\def\rmQ{{\mathbf{Q}}}

\def\rmW{{\mathbf{W}}}
\def\rmX{{\mathbf{X}}}

\DeclareMathAlphabet{\mathsfit}{\encodingdefault}{\sfdefault}{m}{sl}
\SetMathAlphabet{\mathsfit}{bold}{\encodingdefault}{\sfdefault}{bx}{n}

\def\gE{{\mathcal{E}}}

\def\gG{{\mathcal{G}}}

\def\gK{{\mathcal{K}}}

\def\gZ{{\mathcal{Z}}}

\newcommand{\E}{\mathbb{E}}

\newcommand{\bbE}{{\mathbb{E}}}

\newcommand{\bbN}{{\mathbb{N}}}

\newcommand{\bbP}{{\mathbb{P}}}

\newcommand{\bbR}{{\mathbb{R}}}
\newcommand{\bbS}{{\mathbb{S}}}

\newcommand{\calH}{{\mathcal{H}}}

\newcommand{\calN}{{\mathcal{N}}}

\newcommand{\calZ}{{\mathcal{Z}}}

\newcommand{\diag}{\operatorname{diag}}

\newcommand{\la}{\langle}
\newcommand{\ra}{\rangle}

\newcommand{\bitem}{\begin{itemize}}
\newcommand{\eitem}{\end{itemize}}
\newcommand{\benum}{\begin{enumerate}}
\newcommand{\eenum}{\end{enumerate}}
\newcommand{\beq}{\begin{equation}}
\newcommand{\eeq}{\end{equation}}
\newcommand{\beqs}{\begin{equation*}}
\newcommand{\eeqs}{\end{equation*}}

\newcommand{\sgn}{\textnormal{sign}}

\usepackage{color}
\usepackage[linesnumbered,ruled,vlined]{algorithm2e}
\usepackage{algpseudocode}
\usepackage{amsmath,amssymb}
\usepackage{subfigure}
\usepackage{graphicx}

\renewcommand{\rvx}{\bm{x}}
\renewcommand{\rvy}{y}
\renewcommand{\theta}{w}

\newcommand{\proj}{\textnormal{proj}}
\newcommand{\KK}{{K}^{(2)}} 
\newcommand{\hKK}{{\hat{K}}^{(2)}}

\newcommand{\KKK}{{K}^{(3)}}

\newcommand{\KKKK}{{K}^{(4)}}

\newcommand{\rmKK}{{\rmK}^{(2)}} 
\newcommand{\init}{\textnormal{init}} 
\newcommand{\RN}[1]{%
  \textup{\uppercase\expandafter{\romannumeral#1}}%
}

\newcommand{\WWW}[2][]{\rmW^{{#1}(#2)}}
\newcommand{\bbb}[2][]{\rvb^{{#1}(#2)}}
\newcommand{\xxx}[2][]{\rvx^{{#1}(#2)}}

\newcommand{\agnostic}{label-agnostic}

\newcommand{\lantk}{\textsc{LANTK}}

\newcommand{\lantkKR}{\textsc{LANTK-KR}}
\newcommand{\lantkFJLT}{\textsc{LANTK-FJLT}}

\newcommand{\KHR}{K^{(\textnormal{HR})}}
\newcommand{\KNTH}{K^{(\textnormal{NTH})}}

\title{Label-Aware Neural Tangent Kernel: \\ Toward Better Generalization and Local Elasticity}

\author{Shuxiao Chen \quad~~~ Hangfeng He \quad~~~ Weijie J.~Su\\[0.5em]
University of Pennsylvania\\[0.3em]
\texttt{\{shuxiaoc@wharton, hangfeng@seas, suw@wharton\}.upenn.edu}\\
}

\begin{document}

\maketitle


\begin{abstract}
As a popular approach to modeling the dynamics of training overparametrized neural networks (NNs), the neural tangent kernels (NTK) are known to fall behind real-world NNs in generalization ability. This performance gap is in part due to the \textit{label agnostic} nature of the NTK, which renders the resulting kernel not as \textit{locally elastic} as NNs~\citep{he2019local}. In this paper, we introduce a novel approach from the perspective of \emph{label-awareness} to reduce this gap for the NTK. Specifically, we propose two label-aware kernels that are each a superimposition of a label-agnostic part and a hierarchy of label-aware parts with increasing complexity of label dependence, using the Hoeffding decomposition. Through both theoretical and empirical evidence, we show that the models trained with the proposed kernels better simulate NNs in terms of generalization ability and local elasticity.\footnote{Our code is publicly available at \url{https://github.com/HornHehhf/LANTK}.}

\end{abstract}

\section{Introduction}\label{sec:intro}
The last decade has witnessed the huge success of deep neural networks (NNs) in various machine learning tasks \citep{lecun2015deep}. Contrary to its empirical success, however, the theoretical understanding of real-world NNs is still far from complete, hindering its applicability to many domains where interpretability is of great importance, such as autonomous driving and biological research \citep{doshi2017towards}. 


More recently, a venerable line of work relates overparametrized NNs to kernel regression from the perspective of their training dynamics, providing positive evidence towards understanding the optimization and generalization of NNs \citep{jacot2018neural,chizat2018note,cao2019generalization,lee2019wide,arora2019exact,chizat2019lazy,du2019graph,li2019enhanced,zou2020gradient}. To briefly introduce this approach, let $\rvx_i, y_i$ be the feature and label, respectively, of the $i$th data point, and consider the problem of minimizing the squared loss $\frac1n\sum_{i=1}^n (y_i - f(\rvx_i, \rvtheta))^2$ using gradient descent, where $f(\rvx, \rvtheta)$ denotes the prediction of NNs and $\bm w$ are the weights. Starting from an random initialization, researchers demonstrate that the evolution of NNs in terms of predictions  can be well captured by the following kernel gradient descent
\begin{equation}
	\label{eq: grad_flow_f}
	\frac{d}{dt}f(\rvx, \rvtheta_t) = - \frac{1}{n} \sum_{i = 1}^n {K (\rvx, \rvx_i)} \bigl(f(\rvx_i, \rvtheta_t) - \rvy_i\bigr)
\end{equation}
in the infinite width limit, where $\rvtheta_t$ is the weights at time $t$. Above, $K(\cdot, \cdot)$, referred to as \emph{neural tangent kernel} (NTK), is time-independent and associated with the architecture of the NNs. As a profound implication, an infinitely wide NNs is ``equivalent'' to kernel regression with a {deterministic kernel} in the training process.

Despite its immense popularity, many questions still remain unanswered concerning the NTK approach, with the most crucial one, perhaps, being the non-negligible performance gap between a kernel regression using the NTK and a real-world NNs. Indeed, according to an experiment done by \cite{arora2019exact}, on the CIFAR-10 dataset \citep{krizhevsky2009learning}, a vanilla 21-layer CNN can achieve $75.57\%$ accuracy, whereas the corresponding convolutional NTK (CNTK) can only attain $64.09\%$ accuracy. This significant performance gap is widely recognized to be attributed to the \textit{finiteness} of widths in real-world NNs. This perspective has been investigated in many papers \citep{hanin2019finite,huang2019dynamics,bai2019beyond,bai2020taylorized}, where various forms of ``finite-width corrections'' have been proposed. However, the computation of these corrections all involve incremental training. This fact is in sharp contrast to kernel methods, whose computation is ``one-shot'' via solving a simple linear system.


In this paper, we develop a new approach toward closing the gap by recognizing a recently introduced phenomenon termed \textit{local elasticity} in training neural networks~\citep{he2019local}. Roughly speaking, neural networks are observed to be locally elastic in the sense that if the prediction at a feature vector $\bm x$ is not significantly perturbed, after the classifier is updated via stochastic gradient descent at a (labeled) feature vector $\bm x$ that is dissimilar to $\bm x'$ in a certain sense. This phenomenon implies that the interaction between two examples is heavily contingent upon whether their labels are the same or not. Unfortunately, the NTK construction is clearly independent of the labels, thereby being label-agnostic, meaning that it is independent of the labels $\rvy$ in the training data. This ignorance of the label information can cause huge problems in practice, especially when the semantic meanings of the features crucially depends on the \emph{label system}.




\begin{figure}[t]
		\centering
		\scalebox{1.0}{
		\subfigure[NTK: init vs trained]{
			\centering
			\includegraphics[scale=0.22]{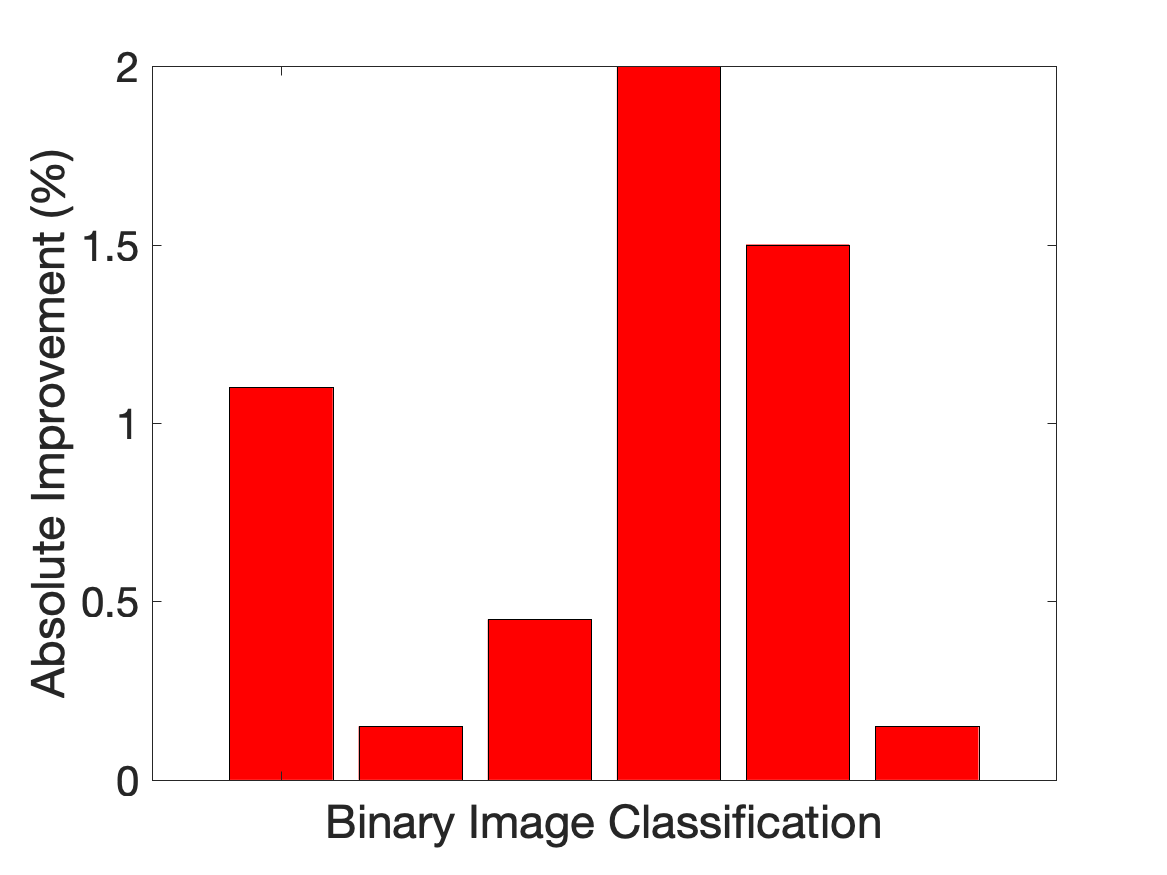}
			\label{fig:ntk-init-trained}}
		\subfigure[Local elasticity: cat vs dog]{
			\centering
			\includegraphics[scale=0.22]{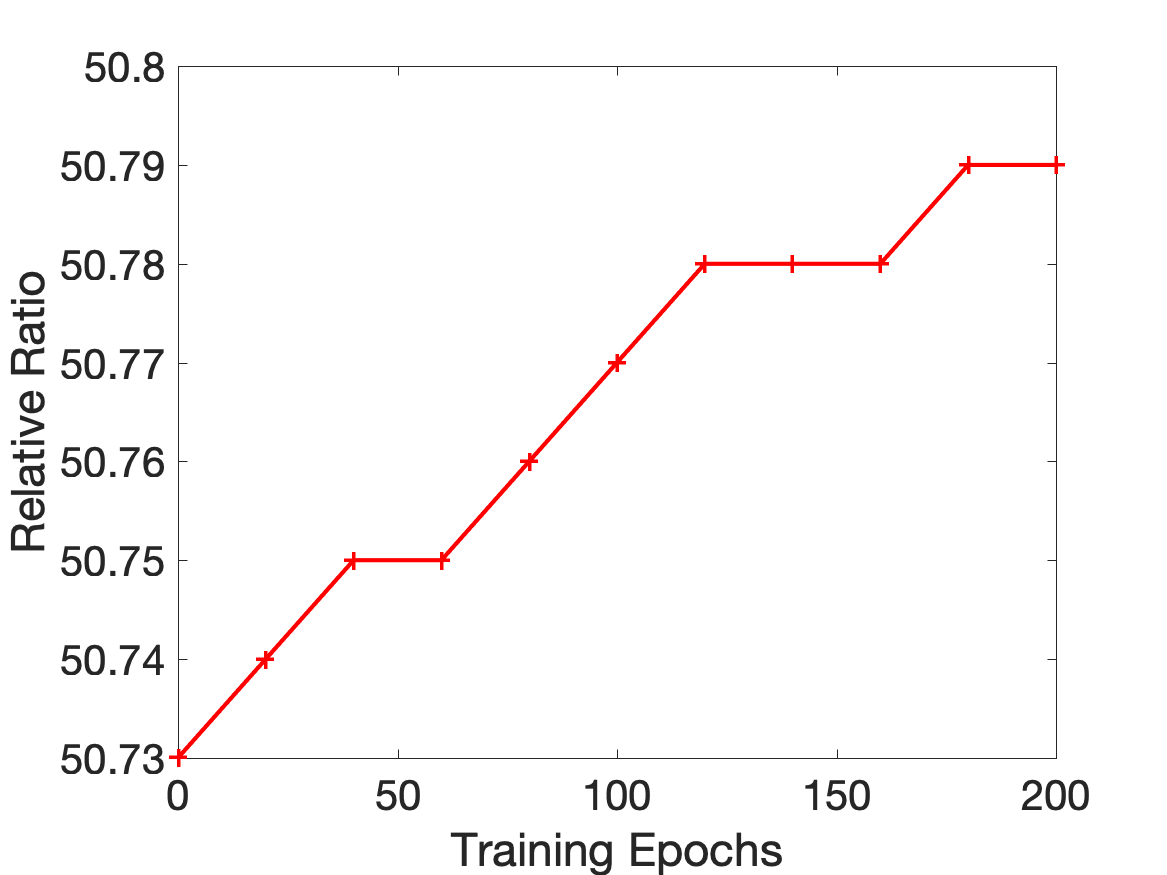}
			\label{fig:le-cat-dog}}
        \subfigure[Local elasticity: chair vs bench]{
			\centering
			\includegraphics[scale=0.22]{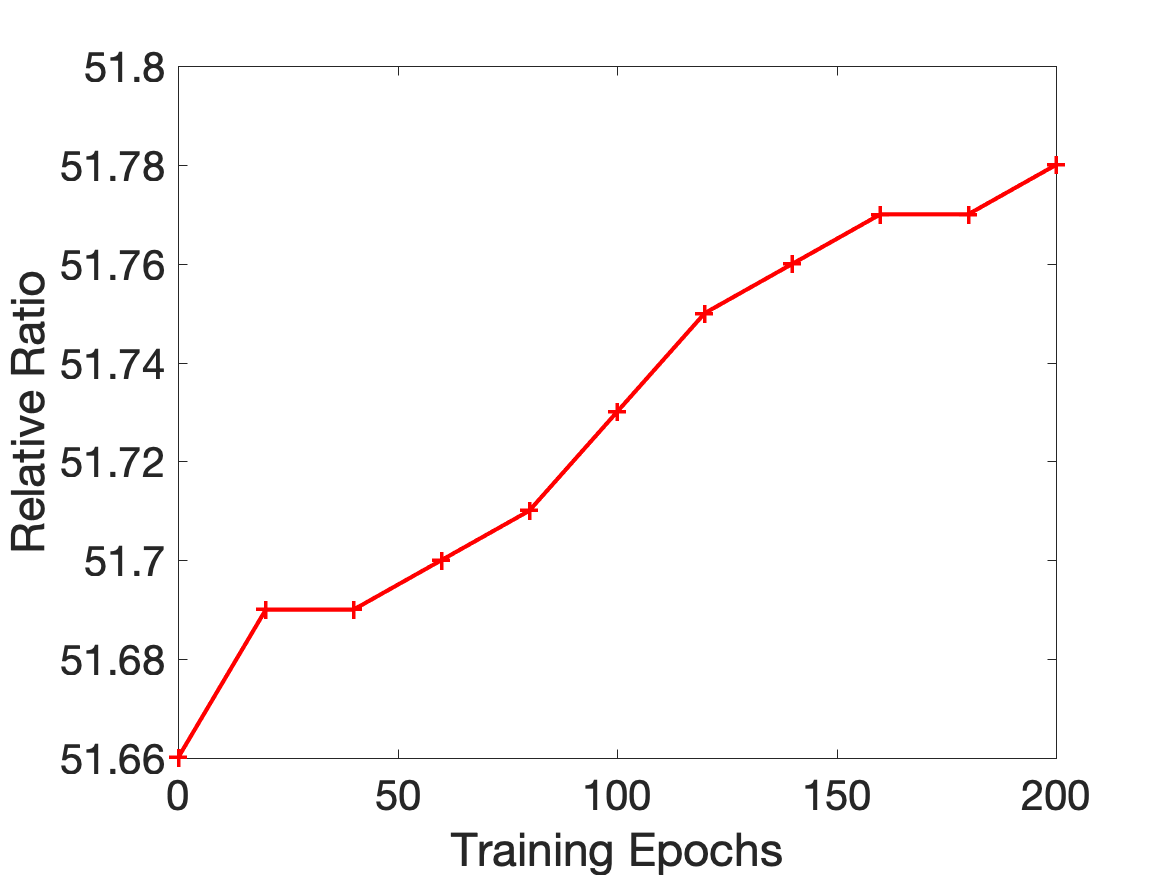}
			\label{fig:le-chair-bench}}
		}
        \caption{The impact of labels on a 2-layer NN in the aspect of generalization ability and local elasticity. Fig. (a) displays the absolute improvement in six binary classification tasks by using $\KK_t$ instead of $\KK_0$, indicating that NNs learn better feature representations with the help of labels along the training process (see Sec.~\ref{subsec:local-elasticity} for more details). Fig. (b) and (c) show that in two label systems (cat v.s. dog and bench v.s. chair), the strength of local elasticity (measured by the relative ratio
        defined in Eq. \eqref{eq: relative_ratio}) both increases with training, but in different patterns, indicating that the labels play an important role in the local elasticity of NNs (See Appx.~\ref{subsec:different-labeling-systems} for more details).}
		\label{fig:labels-impact}
		\vspace{-17pt}
\end{figure}

To shed light on the vital importance of the label information, consider a collection of natural images, where in each image, there are two objects: one object is either a cat or dog, and another is either a bench or chair. 
Take an image $\rvx$ that contains dog+chair, and another $\rvx'$ that contains dog+bench. Then for NTK to work well on both label systems, we would need $\bbE_\init K(\rvx, \rvx')  \approx 1$ if the task is dog v.s. cat, and $\bbE_\init K(\rvx, \rvx')  \approx -1$ if the task is bench v.s. chair,
a fundamental contradiction that cannot be resolved by the label-agnostic NTK (see also Claim \ref{claim: limitations_of_label_agnostic_kernels}). In contrast, NNs can do equally well in both label systems, a favorable property which can be termed adaptivity to label systems. To understand what is responsible for this desired adaptivity, note that \eqref{eq: grad_flow_f} suggests that NNs can be thought of as a time-varying kernel $\KK_t$. This ``dynamic'' kernel differs from the NTK in that it is \emph{label-aware}, because it depends on the trained parameter $\rvtheta_t$ which further depends on $\rvy$. As shown in Fig. \ref{fig:labels-impact}, such an awareness in labels play an import role in the generalization ability and local elasticity of NNs.


Thus, in the search of ``NN-simulating'' kernels, it suffices to limit our focus on the class of label-aware kernels. 
\textcolor{black}{In this paper,} we propose two label-aware NTKs (\lantk s). 
The first one is based on a notion of \emph{higher-order regression}, which extracts higher-order information from the labels by estimating whether two examples are from the same class or not. The second one is based on approximately solving neural tangent hierarchy (NTH), an infinite hierarchy of ordinary different equations that give a precise description of the training dynamics \eqref{eq: grad_flow_f} \citep{huang2019dynamics}, and we show this kernel approximates $\KK_t$ strictly better than $\bbE_\init \KK_0$ does (Theorem \ref{thm: approx_nth_error}). 

\textcolor{black}{Although the two \lantk s stems from very different intuitions, their analytical formulas are perhaps surprisingly similar: they are both quadratic functions of the label vector $y$. This brings a natural question: \emph{is there any intrinsic relationship between the two?} And more generally, \emph{what does a generic label-aware kernel, $K(\rvx, \rvx', S)$, which may have arbitrary dependence structure on the training data $S = \{(\rvx_i, \rvy_i): i\in[n]\}$, look like?}
Using \emph{Hoeffding Decomposition} \citep{hoeffding1948class}, we are able to obtain a structural lemma (Lemma \ref{lemma: hoeffding_decomp}), which asserts that any label-aware kernel can be decomposed into the superimposition of a label-agnostic part and a \emph{hierarchy} of label-aware parts with increasing complexity of label dependence. And the two \lantk s we developed can, in a certain sense, be regarded as the \emph{truncated versions} of this hierarchy. 
} 

We conduct comprehensive experiments to confirm that our proposed \lantk s can indeed better simulate the quantitative and qualitative behaviors of NNs compared to their label-agnostic counterpart. On the one hand, they generalize better: the two \lantk s achieve a $1.04\%$ and $1.94\%$ absolute improvement in test accuracy ($7.4\%$ and $4.1\%$ relative error reductions) in binary and multi-class classification tasks on CIFAR-10, respectively. On the other hand, they are more ``locally elastic'': the relative ratio of the kernelized similarity between intra-class examples and inter-class examples increases, a typical trend observed along the training trajectory of NNs.

\subsection{Related Work}
{\bf Kernels and NNs.} Starting from \citet{neal1996priors}, a line of work considers infinitely wide NNs whose parameters are chosen randomly and only the last layer is optimized \citep{williams1997computing,le2007continuous,hazan2015steps,lee2017deep,matthews2018gaussian,novak2018bayesian,garriga2018deep,yang2019scaling}. 
When the loss is the least squares loss, this gives rise to a class of interesting kernels different from the NTK. 
On the other hand, if all layers are trained by gradient descent, infinitely wide NNs give rise to the NTK
\citep{jacot2018neural,chizat2018note,lee2019wide,arora2019exact,chizat2019lazy,du2019graph,li2019enhanced}, and the NTK also appears implicitly in many works when studying the optimization trajectories of NN training \citep{li2018learning,allen2018convergence,du2018gradient,ji2019polylogarithmic,chen2019much}.

{\bf Limitations of the NTK and corrections.} \cite{arora2019harnessing} demonstrates that in many small datasets, models trained by the NTK can outperform its corresponding NN. But for moderately large scale tasks and for practical architectures, the performance gap between the two is empirically observed in many places and further confirmed by a series of theoretical works \citep{chizat2019lazy,ghorbani2019linearized,yehudai2019power,bietti2019inductive,ghorbani2019limitations,wei2019regularization,allen2019can}. This observation motivates various attempts to mitigate the gap, such as incorporating pooling layers and data augmentation into the NTK \citep{li2019enhanced}, deriving higher-order expansions around the initialization \citep{bai2019beyond,bai2020taylorized}, doing finite-width correction \citep{hanin2019finite}, injecting noise to gradient descent \citep{chen2020generalized}, and most related to our work, the NTH~\citep{huang2019dynamics}. 

{\bf Label-aware kernels.} The idea of incorporating label information into the kernel is not entirely new. For example, \cite{cristianini2002kernel} proposes an alignment measure between two kernels, and argues that aligning a label-agnostic kernel to the ``optimal kernel'' $\rvy\rvy^\top$ can lead to favorable generalization guarantees. This idea is extensively exploited in literature on kernel learning (see, e.g., \citealt{lanckriet2004learning,cortes2012l2,gonen2011multiple}). 
In another related direction, \cite{tishby2000information} proposes the information bottleneck principle, which roughly says that an optimal feature map should simultaneously minimize its mutual information (MI) with the feature distribution and maximize its MI with the label distribution, thus incorporating the label information (see also \citealt{tishby2015deep} in the context of deep learning). 
We refer the readers to Appx. \ref{append:connections} for a detailed discussion of the connections and differences of our proposals to these two lines of research.


\subsection{Preliminaries}

We focus on binary classification tasks for ease of exposition, but our results can be easily extend to multi-class classification tasks. 
Suppose we have i.i.d. data $S = \{(\rvx_i, \rvy_i): i\in [n]\} \subseteq \mathbb{R}^d \times \{\pm 1\}$. 
Let $\rmX \in \mathbb{R}^{n\times d}$ be the feature matrix and let $\rvy \in \{\pm 1\}^n$ be the vector of labels.
Let $\rvtheta = \{(\rmW ^{(\ell)}, \rvb^{(\ell)})\in \bbR^{p_\ell \times q_\ell}\times\bbR^{q_\ell}: \ell \in [L]\}$ be the collection of weights and biases at each layer. A neural network function is recursively defined as $\xxx{0} = \rvx, \xxx{\ell}  = \frac{1}{\sqrt{q_\ell}} \sigma(\WWW{\ell} \xxx{\ell-1} + \bbb{\ell})$ and $f(\rvx, \rvtheta)	 = \xxx{L}$, 
where $\sigma$ is the activation function which applies element-wise to matrices. Note that $q_1 = d$ and $p_L = 1$. 


Consider training NNs with least squares loss: $L(\rvtheta) = \frac{1}{2n} \sum_{i=1}^n \bigl(f(\rvx_i, \rvtheta)- \rvy_i\bigr)^2$.
The gradient flow dynamics (i.e., gradient descent with infinitesimal step sizes) is given by $\dot{\rvtheta_t} = - \frac{\partial}{\partial \rvtheta} L(\rvtheta_t),$
where we use the dot notation to denote the derivative w.r.t. time, and $\rvtheta_t$ is the parameter value at time $t$. The dynamics w.r.t. $f$ is given by a kernel gradient descent \eqref{eq: grad_flow_f}, 
where the corresponding \emph{second-order} kernel is given by $\KK_t(\rvx, \rvx') =  \la \frac{\partial}{\partial \rvtheta} f(\rvx, \rvtheta_t), \frac{\partial}{\partial \rvtheta} f(\rvx', \rvtheta_t)  \ra$.
With sufficient over-parameterization, one can show that $\KK_t \approx \KK_0$, and with Gaussian initialization, one can show that $\KK_0$ concentrates around $\bbE_\init \KK_0$, {where $\bbE_{\init}$ is the expectation operator w.r.t. the random initialization}. Hence, $\KK_t$ can be approximated by the \emph{deterministic kernel} $\bbE_\init \KK_0$, the NTK.

We have heuristically argued in Section \ref{sec:intro} that the label-agnostic property of NTK can cause problems when the semantic meaning of the features is highly dependent on the label system. To further illustrate this point, in Appx.~\ref{subappend:example_limit_of_label_agnostic_kernels}, we construct a simple example that validates the following claim: 
\begin{claim}[Curse of Label-Agnosticism]
\label{claim: limitations_of_label_agnostic_kernels}
If a kernel $K$ is label-agnostic and it works well on one label system, then there exists a natural relabeling, which gives another label system on which $K$ performs arbitrarily bad. In other words, $K$ is not adaptive to different label systems.
\end{claim}

We make two remarks. The above claim applies to any label-agnostic kernels, and the NTK is only a specific example. 
Moreover, the above claim only applies to the generalization error. A label-agnostic kernel can always achieve zero training error, as long as the corresponding kernel matrix is invertible. 

\section{Construction of Label-Aware NTKs}\label{sec:label_aware_kernels}


We now propose two ways to incorporate label awareness into NTKs. In Section~\ref{sec:conn-with-hoeffd}, we give a unified interpretation of the two approaches using the Hoeffding decomposition.

\subsection{Label-Aware NTK via Higher-Order Regression}\label{subsec:higher-order-regression}
We shall all agree that a ``good'' feature map should map the feature vector $\rvx$ to $\phi(\rvx )$ s.t. a simple linear fit in the $\phi(\cdot)$-space can give satisfactory performance. 
In this sense, the ``optimal'' feature map is obviously its corresponding label: $\phi(\rvx ) := \textnormal{label of } \rvx$. Hence, for $\alpha, \beta\in[n]$, the corresponding ``optimal'' kernel is given by $K^{\star}(\rvx_\alpha, \rvx_\beta) = \rvy_\alpha \rvy_\beta.$
This motivates us to consider interpolating the NTK with the optimal kernel:
$
\bbE_\init K^{(2)}_0(\rvx_\alpha, \rvx_\beta) + \lambda K^{\star}(\rvx_\alpha, \rvx_\beta) = \bbE_\init K^{(2)}_0(\rvx_\alpha, \rvx_\beta)+ \lambda \rvy_\alpha \rvy_\beta, 
$
where $\lambda \geq 0$ controls the strength of the label information. 
However, this interpolated kernel cannot be calculated on the test data. A natural solution is to estimate $\rvy_\alpha \rvy_\beta$. Thus, we propose to use the following kernel, which we term as \lantk-HR:
\begin{equation}
	\label{eq: 2nd_order_reg_kernel}
	K^{(\textnormal{HR})}(\rvx, \rvx') :  = \bbE_\init K^{(2)}_0(\rvx, \rvx') + \lambda \gZ(\rvx, \rvx', S), 
\end{equation}
where $\gZ(\rvx, \rvx', S)$ is an estimator of $(\textnormal{label of }\rvx) \times (\textnormal{label of }\rvx')$ --- a quantity indicating whether both features belong to the same class or not --- using the dataset $S$.

A natural way to obtain $\gZ$ is to form a new \emph{pairwise dataset} $\bigl\{\bigl( (\rvx_i, \rvx_j), \rvy_i\rvy_j\bigr): i, j\in[n]\bigr\}$, where we think of $\rvy_i\rvy_j$ as the label of the augmented feature vector $(\rvx_i, \rvx_j)$. If we use a linear regression model, then we would have $\gZ(\rvx, \rvx', S) = \rvy^\top \rmM(\rvx, \rvx', \rmX) \rvy$ for some matrix $\rmM\in \mathbb{R}^{n\times n}$ depending on the two feature vectors $\rvx, \rvx'$ as well as the feature matrix $\rmX$, hence the name ``higher-order regression''. 

What requirement we do need to impose on the estimator $\gZ$? Heuristically, the inclusion of $\gZ$ is only helpful when additional information in the training data is ``extracted'' apart from the one extracted by the label-agnostic part. We refer the readers to Sec.~\ref{subsec:kernel-regression} and Appx.~\ref{subsec:details-of-Z} for more details on the choice of $\gZ$.


\subsection{Label-Aware NTK via Approximately Solving NTH}
As we have discussed in Section \ref{sec:intro}, different from the label-agnostic NTK $\bbE_\init \KK_0$, the time-dependent kernel $\KK_t$ is indeed label-aware. This suggests that in real-world neural network training, $\KK_t$ must drift away from $\KK_0$ by a non-negligible amount. Indeed, \cite{huang2019dynamics} prove that the second order kernel $\KK_t$ evolves according to the following infinite \emph{hierarchy} of ordinary differential equations (that is, the NTH): 
\begin{align}
	\label{eq: nth}
	\dot{K}^{(r)}_t(\rvx_{\alpha_1}, \cdots, \rvx_{\alpha_r}) & = -\frac{1}{n} \sum_{i\in[n]} K^{(r+1)}_t(\rvx_{\alpha_1}, \cdots, \rvx_{\alpha_r}, \rvx_i)(f_t(\rvx_i) - \rvy_i), \ \ \ r\geq 2,
\end{align}
which, along with \eqref{eq: grad_flow_f}, gives a \emph{precise description} of the training dynamics of NNs.
Here, $K_t^{(r)}$ is an \emph{$r$-th order kernel} which takes $r$ feature vectors as the input. 



Our second construction of \lantk~is to obtain a finer approximation of $\KK_t$. Let $\rvf_t$ be the vector of $f_t(\rvx_i)$'s. Using \eqref{eq: nth}, we can re-write 
$\KK_t(\rvx, \rvx')$ as $\KK_0(\rvx, \rvx') -\frac{1}{n} \int_{0}^t \langle \KKK_0(\rvx, \rvx', \cdot), \rvf_u-\rvy \rangle du + \frac{1}{n^2} \int_0^t \int_0^u (\rvf_u - \rvy)^\top \KKKK_v(\rvx, \rvx', \cdot, \cdot) (\rvf_v - \rvy) dvdu,$
where $\KKK_t(\rvx,\rvx', \cdot)$ is an $n$-dimensional vector whose $i$-th coordinate is $\KKK_t(\rvx,\rvx', \rvx_i)$, and $\KKKK_t(\rvx,\rvx', \cdot,\cdot)$ is an $n\times n$ matrix, whose $(i, j)$-th entry is $\KKKK_t(\rvx,\rvx', \rvx_i, \rvx_j)$.
Now, let $\rmKK_0 \in \mathbb{R}^{n\times n}$ be the kernel matrix corresponding to $\KK_0$ computed on the training data. The kernel gradient descent using $\rmKK_0$ is characterized by $\dot{\rvh}_t = -\frac{1}{n} \rmKK_0(\rvh_t - \rvy)$ with the initial value condition $\rvh_0 = \rvf_0$,
whose solution is given by $\rvh_t = (\rmI_n - e^{-t \rmKK_0/n}) \rvy + \rvh_0.$
For a wide enough network, we expect $\rvh_t \approx \rvf_t$, at least when $t$ is not too large. On the other hand, it has been shown in \cite{huang2019dynamics} that $\KKKK_t$ varies at a rate of $\tilde{O}(1/m^2)$, where $\tilde O$ hides the poly-log factors and $m$ is the (minimum) width of the network. Hence, it is reasonable to approximate $\KKKK_t$ by $\KKKK_0$. This motivates us to write
\begin{align}
    \label{eq: approx_nth_full}
	\KK_t(\rvx, \rvx') & = \hat{K}^{(2)}_t(\rvx, \rvx') + \gE 
\end{align}
where $\gE$ is a small error term and $\hat{K}^{(2)}_t(\rvx, \rvx') = \KK_0(\rvx, \rvx') -\frac{1}{n} \int_{0}^t \langle \KKK_0(\rvx, \rvx', \cdot), \rvh_u-\rvy \rangle du  + \frac{1}{n^2} \int_0^t \int_0^u (\rvh_u - \rvy)^\top \KKKK_0(\rvx, \rvx', \cdot, \cdot) (\rvh_v - \rvy) dvdu$.



In Appx.~\ref{subappend:proof_approx_nth_zero_init}, we show that $\hKK_t$ can be evaluated analytically, and the analytical expression is of the form $\KK_0 + \rva^\top \rvy + \rvy^\top \rmB \rvy$ for some vector $\rva$ and matrix $\rmB$ (see Proposition \ref{prop: approx_nth_zero_init} for the exact formula). 
Moreover, the approximation \eqref{eq: approx_nth_full} can be justified formally under mild regularity conditions: 
\begin{theorem}[Validity of $\hKK_t$]
	\label{thm: approx_nth_error}
	Consider a neural network with $\rvb^{(\ell)} = 0 \ \forall \ell$ and $p_\ell = q_\ell = m$ except for $q_1 = d, p_L = 1$. We train this net using gradient flow with the least squares loss and Gaussian initialization. Under Assumption \ref{assump: justify_approx_error} in Appendix \ref{subappend:proof_approx_nth_error}, with probability at least $1-e^{-Cm}$ w.r.t. the initialization, the error of the approximation \eqref{eq: approx_nth_full} at time $t$ satisfies 
	\[
		|\gE| \lesssim \frac{(\log m)^c}{m^2} \times\bigg[ t^2(t^2\lor 1) \bigg(1+\frac{t^2}{m}\bigg)\bigg(t \land \frac{n}{\lambda}\bigg) + t^3(1\lor t) \bigg] \textnormal{ provided }  t \lesssim \bigg({\frac{\sqrt{\lambda m/n}}{(\log m)^{c'}}} \land \frac{m^{1/3}}{(\log m)^{c''}}\bigg),
	\]
	where $C, c, c', c''$ are some absolute constants and $\lambda$ is a quantity defined in Asumption \ref{assump: justify_approx_error}. On the other hand, on the same high probability event, the three terms in $\hat{K}^{(2)}_t$ are at most $\tilde{O}(1), \tilde{O}(t/m)$ and $\tilde{O}(t^2/m),$
	respectively.
\end{theorem}

By the above theorem, the error term is much smaller than the main terms for large $m$, justifying the validity of the approximation \eqref{eq: approx_nth_full}. Meanwhile, this approximation is \emph{strictly better} than $\bbE_\init \KK_0$, whose error term is shown to be $\tilde O(1/m)$ in \cite{huang2019dynamics}.

Based on the approximation \eqref{eq: approx_nth_full}, we propose the following \lantk-NTH:
\begin{align}
	K^{(\textnormal{NTH})}(\rvx, \rvx') &:= \bbE_\init \KK_0(\rvx, \rvx') + \rvy^\top (\E_\init\rmKK_0)^{-1} \E_\init[\KKKK_0(\rvx, \rvx', \cdot, \cdot)] (\E_\init\rmKK_0)^{-1} \rvy \nonumber\\
	\label{eq: approx_nth_kernel}
	& \qquad - \rvy^\top \bar\rmP\bar\rmQ(\rvx, \rvx')\bar\rmP^\top \rvy,
\end{align}
where $\bar \rmP\bar \rmD\bar \rmP^\top$ is the eigen-decomposition of $\bbE_\init \rmK_0^{(2)}$ and the $(i, j)$-th entry of $\bar\rmQ(\rvx, \rvx')$ is given by ${\bigl(\bar\rmP^\top \bbE_\init[\KKKK_0(\rvx, \rvx', \cdot, \cdot)] \bar\rmP \bar\rmD^{-1}\bigr)_{ij}}\bigr/({\bar\rmD_{ii} + \bar\rmD_{jj}}).$ 
In a nutshell, we take the formula for $\hKK$, integrate it w.r.t. the initialization, and send $t\to \infty$. The linear term $\rva^\top \rvy$ vanishes as $\bbE_\init [\rva] = 0$.\footnote{We can in principle prove a similar result as Theorem \ref{thm: approx_nth_error} for $\KNTH$ by exploiting the concentration of $\KKK_0$ and $\KKKK_0$, but since it is not the focus of this paper, we omit the details.}




\subsection{Understanding the connection between \lantk-HR and \lantk-NTH}
\label{sec:conn-with-hoeffd}
\textcolor{black}{The formulas for the two \lantk s are perhaps unexpectedly similar --- they are both quadratic functions of $\rvy$ (here we focus on linear regression based $\gZ$ in $\KHR$). 
Is there any intrinsic relationship between the two? In this section, we propose to understand their relationship via a classical tool from asymptotic statistics called \emph{Hoeffding decomposition} \citep{hoeffding1948class}. Let us start by considering a generic label-aware kernel $K(\rvx, \rvx') \equiv K(\rvx, \rvx', S)$, which may have arbitrary dependence on the training data $S$.}
For a fixed $A \subseteq [n]$, define
\begin{equation}
    \label{eq: func_space_hoeffding_decomp}
    \gG_A := \bigg\{g: \rvy_A\mapsto f(\rvy_A) \ \bigg| \ \E_{\rvy|\rmX} g(\rvy_A)^2 < \infty, \E_{\rvy|\rmX}[g(\rvy_A) | \rvy_B] = 0 \ \forall B \textnormal{ s.t. } |B| < |A|\bigg\},
\end{equation}
The requirement of $\E_{\rvy|\rmX}[g(\rvy_A) | \rvy_B] = 0 \ \forall B \textnormal{ s.t. } |B| < |A|$ means that for any $B$ whose cardinality is less than $A$, the function $g$ has no information in $\rvy_B$, which reflects our intention to orthogonally decompose a generic function of $\rvy$ into parts whose dependence on $\rvy$ is increasingly complex. 

The celebrated work of \cite{hoeffding1948class} tells that the function spaces $\{\gG_A : A \subseteq [n]\}$ form an orthogonal decomposition of  $\sigma(\rvy \ | \ \rmX)$, the space of all measurable functions of $\rvy$ (w.r.t. the conditional law of $\rvy \ | \ \rmX$). Specifically, we have:
\begin{lemma}[Hoeffding decomposition of a generic label-aware kernel]
\label{lemma: hoeffding_decomp}
We can decompose a generic label-aware $K$ into
\begin{equation}
	\label{eq: hoeffding_decomp}
	K(\rvx, \rvx', S)  = \sum_{A \subseteq[n]}  \proj_AK(\rvx, \rvx', S) = \sum_{r = 0}^n \sum_{A \subseteq[n]: |A| = r}  \proj_AK(\rvx, \rvx', S),
\end{equation}
where $\proj_A$ is the $L^2$ projection operator onto the function space $\gG_A$. In particular, we can write 
\begin{align}
\label{eq: hoefdding_decomp_layerwise} 
	K(\rvx, \rvx', S) & =  
	\gK^{(2)}_\varnothing(\rvx, \rvx') + \sum_{i\in[n]}\gK^{(3)}_i(\rvx, \rvx', \rvy_i) + \sum_{i\neq j}\gK^{(4)}_{ij}(\rvx, \rvx', \rvy_i, \rvy_j) + \cdots ,
\end{align}
where $\{\gK^{(r+2)}_{A}: 0\leq r\leq n , |A| = r\}$ is a collection of measurable functions s.t. $\gK^{(r+2)}_A$ only depends on $\rvy_A$.
\end{lemma}
Thus, we have decomposed a generic label-aware kernel into the superimposition of a hierarchy of smaller kernels with \emph{increasing complexity} on the label dependence structure: at the zero-th level, $\gK^{(2)}_\varnothing$ is label-agnostic; at the first level, $\{\gK^{(3)}_i\}$ depend only on a single coordinate of $\rvy$; at the second level, $\{\gK^{(4)}_{ij}\}$ depend on a pair of labels, etc. Moreover, \emph{the information at each level is orthogonal}. 

\textcolor{black}{In view of the above lemma, the two \lantk s we proposed can be regarded as \emph{truncated versions} of \eqref{eq: hoefdding_decomp_layerwise}, where the truncation happens at the second level.}

\textcolor{black}{The above observation motivates us to ask a more ``fine-grained'' question regarding the two \lantk s: \emph{is the label-aware part in $\KNTH$ implicitly doing a higher-order regression?}}

\textcolor{black}{To offer some intuitions for this question, we randomly choose $5$ planes and $5$ ships in CIFAR-10\footnote{We only sample such a small number of images because the computation cost of $\KNTH$ is at least $O(n^4)$.}, and we compute the intra-class and inter-class values of the label-aware component in a two-layer fully-connected $\KNTH$ (See Appx.~\mbox{\ref{append:nth_calculations}} for the exact formulas). We find that the mean of the intra-class values is $98$, whereas the mean of the inter-class values is $48$. The difference between the two means indicates that $\KNTH$ may implicitly try to increase the similarity between two examples if they come from the same class and decrease it otherwise, and this behavior agrees with the \emph{intention} of $\KHR$.}

\textcolor{black}{We conclude this section by remarking that our above arguments are, of course, largely heuristic, and a rigorous study on the relationship between the two \lantk s is left as future work.}

\section{Experiments}\label{sec:experiment}

Though we have proved in Theorem \mbox{\ref{thm: approx_nth_error}} that $\KNTH$ possesses favorable theoretical guarantees, it's computational cost is prohibitive for large-scale experiments. Hence, in the rest of this section, all experiments on \lantk s are based on $\KHR$. However, in view of their close connections established in Sec. \ref{sec:conn-with-hoeffd}, we conjecture that similar results would hold for $\KNTH$.

\subsection{Generalization Ability}
\label{subsec:kernel-regression}
We compare the generalization ability of the \lantk~to its \agnostic~counterpart in both binary and multi-class image classification tasks on CIFAR-10.
We consider an $8$-layer CNN and its corresponding CNTK. The implementation of CNTK is based on \cite{neuraltangents2020} and the details of the architecture can be found in Appx. \ref{subsec:CNN-architecture}. 

{\bf Choice of higher-order regression methods.} 
Due to the high computational cost of kernel methods, our choice of $\gZ$ is particularly simple. We consider two methods: one is a kernelized regression (\lantkKR-V1 and V2\footnote{V1 and V2 stand for two ways of extracting features.}), and the other is a linear regression with Fast-Johnson-Lindenstrauss-Transform (\lantkFJLT-V1 and V2), a sketching method that accelerates the computation \citep{ailon2009fast}. We refer the readers to Appx.~\ref{subsec:details-of-Z} for further details. The choice of ``fancier'' $\gZ$ is deferred to future work.

\begin{table*}[t]
\centering
\scalebox{0.8}{
\begin{tabular}{c||c|c|c|c|c||c|c}
\Xhline{2\arrayrulewidth}
 & {deer vs dog} & {cat vs deer}  & { cat vs frog} & { deer vs frog} & {bird vs frog} & Avg & Imp. (abs./rel.) \\ \hline \hline
 CNTK & 85.15 & 83.55 & 86.95 & 87.55& 86.35 & 85.91 & -\\
 \lantk-best & {\bf 86.75}& {\bf 84.85} & {\bf 87.40} & {\bf 88.30} & {\bf 87.45} & {\bf 86.95}& 1.04 / 7.4\% \\ \hline \hline
 \lantkKR-V1 & 85.65 & 83.90 & 87.20 & 87.85& {\bf 87.45} & 86.41 & 0.50 / 3.5\%\\
 \lantkKR-V2 & 85.80 & 83.90 & 87.15& 87.90 & 87.25 & 86.40 & 0.49 / 3.5\%  \\
 \lantkFJLT-V1 & {\bf 86.75} & {\bf 84.85} & {\bf 87.40} & 88.10 & 86.40 & {\bf 86.70} & 0.79 / 5.6\% \\
 \lantkFJLT-V2 & 86.25& 84.55 & 87.10 & {\bf 88.30} & 87.00 & 86.64& 0.73 / 5.2\% \\ \hline \hline
 CNN & 89.00 & 88.50 & 89.00 & 92.50 & 90.15 & 89.83 & 3.92 / 27.8\% \\
 \Xhline{2\arrayrulewidth}
\end{tabular}}
\caption{Performance of \lantk~, CNTK and CNN on binary image classification tasks on CIFAR-10. 
Note that the best \lantk (indicated as \lantk-best) significantly outperforms CNTK. 
Here, ``abs'' stands for absolute improvement and ``rel'' stands for the relative error reduction. 
}
\label{table:binary-kernel-regression}
\end{table*}

{\bf Binary classification.} We first choose five pairs of categories on which the performance of the 2-layer fully-connected NTK is neither too high (o.w. the improvement will be marginal) nor too low (o.w. it may be difficult for $\gZ$ to extract useful information). 
We then randomly sample $10000$ examples as the training data and another $2000$ as the test data, under the constraint that the sizes of positive and negative examples are equal.
The performance on the five binary classification tasks are shown in Table \ref{table:binary-kernel-regression}. 
We see that \lantkFJLT-V1 works the best overall, followed by \lantkFJLT-V2 (which uses more features).
The improvement compared to CNTK indicates the importance of label information and confirms our claim that \lantk~better simulates the quantitative behaviors of NNs.

\begin{table*}[t]
\centering
\scalebox{0.8}{
\begin{tabular}{c||c|c|c|c||c|c}
\Xhline{2\arrayrulewidth}
Training size & 2000 & 5000  & 10000 & 20000 &  Avg & Imp. (abs./rel.) \\ \hline \hline
 CNTK &44.01 & 50.44& 55.18& 60.12& 52.44 & - \\
 \lantk-best & {\bf 46.31} & {\bf 52.66} & {\bf 56.96} & {\bf 61.58} & {\bf 54.38} & 1.94 / 4.1\%  \\ \hline \hline
 \lantkKR-V1  &44.49 & 51.39 & 55.74 & 60.87& 53.12& 0.68 / 1.4\%\\
 \lantkKR-V2  & 44.64& 51.38& 55.89 & 60.87& 53.20 & 0.76 / 1.6\%   \\
 \lantkFJLT-V1  &{\bf 46.31} & {\bf 52.66} & {\bf 56.96} & {\bf 61.58} & {\bf 54.38} & 1.94 / 4.1\%  \\
 \lantkFJLT-V2  & 45.72& 52.50&  56.26& Failed &- & 1.62 / 3.4\%  \\ \hline \hline
 CNN  &45.30 & 52.70& 60.23 & 68.15& 56.60 & 4.16 / 8.7\% \\
 \Xhline{2\arrayrulewidth}
\end{tabular}}
\caption{Performance of \lantk~, CNTK, and CNN on multi-class image classification tasks on CIFAR-10. The improvement is  more evident than the binary classification.
}
\label{table:multiple-kernel-regression}
\end{table*}

{\bf Multi-class classification.} Our current construction of \lantk~is under binary classification tasks, but it can be easily extended to multi-class classification tasks by changing the definition of the ``optimal'' kernel to $K^{\star}(\rvx_\alpha, \rvx_\beta) = \mathbf{1}\{\rvy_\alpha=\rvy_\beta\}$. Here, we again samples from CIFAR-10 under the constraint that the classes are balanced, and we vary the size of the training data in $\{2000, 5000, 10000, 20000\}$ while fixing the size of the test data to be $10000$. 
The results are shown in Fig. \ref{table:multiple-kernel-regression}. ``Failed'' in the table means that the experiment is failed due to memory constraint.
Similarly, \lantkFJLT-V1 works the best among all training sizes, and the performance gain is even higher compared to binary classification tasks. This supports our claim that the label information is particularly important when the semantic meanings of the features are highly dependent on the label system.

\subsection{Local Elasticity}
\label{subsec:local-elasticity}
Local elasticity, originally proposed by \citet{he2019local}, refers to the phenomenon that the prediction of a feature vector $\rvx'$ by a classifier (not necessarily a NN) is not significantly perturbed, after this classifier is updated via stochastic gradient descent at another feature vector $\rvx$, if $\rvx$ and $\rvx'$ are \emph{dissimilar} according to some geodesic distance which captures the semantic meaning of the two feature vectors. 
Through extensive experiments, \citet{he2019local} demonstrate that NNs are locally elastic, whereas linear classifiers are not. Moreover, they show that models trained by NTK is far less locally elastic compared to NNs, but more so compared to linear models. 

In this section, we show that \lantk~is significantly more locally elastic than NTK. We follow the experimental setup in \citet{he2019local}.
Specifically, for a kernel $K$, define the \emph{normalized kernelized similarity}\footnote{This similarity measure is closely related stiffness \citep{fort2019stiffness}, but the difference lies in that stiffness takes the loss function into account.} as $\bar{K}(\rvx, \rvx'):= K(\rvx, \rvx')/\sqrt{K(\rvx, \rvx) K(\rvx', \rvx')}$. Then, the strength of local elasticity of the corresponding kernel regression can be quantified by the relative ratio of $\bar K$ between intra-class examples and inter-class examples:
\begin{equation}
    \label{eq: relative_ratio}
    \textnormal{RR}(K):= \frac{\sum_{(\rvx, \rvx')\in S_1} \bar K(\rvx, \rvx') / |S_1|}{\sum_{(\rvx, \rvx')\in S_1} \bar K(\rvx, \rvx') / |S_1| + \sum_{(\tilde\rvx, \tilde\rvx')\in S_2} \bar K(\tilde\rvx, \tilde\rvx') / |S_2|},
\end{equation}
where $S_1$ is the set of intra-class pairs and $S_2$ is the set of inter-class pairs. Note that under this setup, the strength of local elasticity of NNs corresponds to $\textnormal{RR}(\KK_t)$. In practice, we can take the pairs to both come from the training set (train-train pairs), or one from the test set and another from the training set (test-train pairs).


\begin{table*}[t]
\centering
\scalebox{0.8}{
\begin{tabular}{c||c|c|c|c|c|c}
\Xhline{2\arrayrulewidth}
 Train-train & frog vs ship& frog vs truck & deer vs ship & dog vs truck & bird vs truck & deer vs truck  \\ \hline
NN-init & 58.37 & 55.07& 57.50 & 54.75 & 52.93& 54.86 \\ 
NN-trained & {\bf 71.99 } $\uparrow$  & {\bf 68.36 }  $\uparrow$ & {\bf 69.98}  $\uparrow$& {\bf 66.35}  $\uparrow$& {\bf 63.99}  $\uparrow$& {\bf 65.96}  $\uparrow$\\ \hline
NTK & 63.83 &58.31 & 62.43 & 58.05 & 55.01 & 58.02 \\ 
\lantk & {\bf 66.62}  $\uparrow$& {\bf 60.57}  $\uparrow$& {\bf 64.90}  $\uparrow$& {\bf 59.75}  $\uparrow$& {\bf 55.94}  $\uparrow$& {\bf 59.59}  $\uparrow$\\ 
\hline \hline
Test-train & frog vs ship& frog vs truck & deer vs ship & dog vs truck & bird vs truck & deer vs truck  \\ \hline 
NN-init & 58.31 & 55.06 & 57.64 & 54.62 & 52.93 & 54.94 \\ 
NN-trained & {\bf 71.45}  $\uparrow$& {\bf 67.91}  $\uparrow$& {\bf 69.73} $\uparrow$& {\bf 65.80}  $\uparrow$& {\bf 63.58} $\uparrow$ & {\bf 65.53}  $\uparrow$\\ \hline
NTK & 63.76 & 58.30 & 62.67 & 57.84 & 55.00 & 58.14 \\ 
\lantk & {\bf 66.53} $\uparrow$& {\bf 60.08} $\uparrow$& {\bf 65.20} $\uparrow$ & {\bf 59.54} $\uparrow$& {\bf 55.97} $\uparrow$ & {\bf 59.77} $\uparrow$ \\ 
 \Xhline{2\arrayrulewidth}
\end{tabular}
}
\caption{Strength of local elasticity in binary classification tasks on CIFAR-10. 
The training makes NNs more locally elastic, and \lantk~successfully simulates this behavior.}
\label{table:label-le}
\end{table*}

We first compute the strength of local elasticity for two-layer NNs with $40960$ hidden neurons at initialization and after training. We find that for all $\binom{10}{2} = 45$ binary classification tasks in CIFAR-10, $\textnormal{RR}$ significantly increases after training, under both train-train and test-train settings, agreeing with the result in Fig.~\ref{fig:le-cat-dog} and Fig.~\ref{fig:le-chair-bench}. The top $6$ tasks with the largest increase in $\textnormal{RR}$ in Table \ref{table:label-le}. The absolute improvement in the test accuracy is shown in Fig.~\ref{fig:ntk-init-trained}.

We then compute the strength of local elasticity for the corresponding NTK and \lantk~(specifically, \lantk-KR-V1) based on the formulas derived in Appx. \ref{subsubsection:label-agnostic-kernel},
and the results for the same $6$ binary classification tasks are shown in Table \ref{table:label-le}. We find that \lantk~is more locally elastic than NTK, indicating that \lantk~better simulates the qualitative behaviors of NNs.



\section{Conclusion}\label{sec:discussion}

In this paper, we proposed the notion of label-awareness to explain the performance gap between a model trained by NTK and real-world NNs. 
Inspired by the Hoeffding Decomposition of a generic label-aware kernel, we proposed two label-aware versions of NTK, both of which are shown to better simulate the behaviors of NNs via a theoretical study and comprehensive experiments. 

\textcolor{black}{We conclude this paper by mentioning several potential future directions.}


\textcolor{black}{{\bf More efficient implementations.} Our implementation of $K^{(\textnormal{HR})}$ requires forming a pairwise dataset, which can be cumbersome in practice (indeed, we need to use fast FJLT to accelerate the computation). Moreover, the exact computation of $K^{(\textnormal{NTH})}$ requires at least $O(n^4)$ time, since the dimension of the matrix $\bbE_\init [\KKKK_0(\rvx, \rvx', \cdot, \cdot)]$ is $n^2\times n^2$. It would greatly improve the practical usage of our proposed kernels if there are more efficient implementations.}

\textcolor{black}{{\bf Higher-level truncations.} As discussed in Sec. \ref{sec:conn-with-hoeffd}, our proposed $K^{(\textnormal{HR})}$ and $K^{(\textnormal{NTH})}$ can be regarded as \emph{second-level truncations} of the Hoeffding Decomposition \eqref{eq: hoeffding_decomp}. In principle, our constructions can be generalized to higher-level truncations, which may give rise to even better ``NN-simulating'' kernels. However, such generalizations would incur even higher computational costs. It would be interesting to see even such generalizations can be done with a reasonable amount of computational resources.}

\textcolor{black}{{\bf Going beyond least squares and gradient flows.} Our current derivation is based on a neural network trained by squared loss and gradient flows. While such a formulation is common in the NTK literature and makes the theoretical analysis simpler, it is of great interest to extend the current analysis to more practical loss functions and optimization algorithms.}

\section*{Acknowledgments}

This work was in part supported by NSF through CAREER DMS-1847415 and CCF-1934876, an Alfred Sloan Research Fellowship, the Wharton Dean's Research Fund, and Contract FA8750-19-2-0201 with the US Defense Advanced Research Projects Agency (DARPA). 


\section*{Broader Impact}


While this work may have certain implications on the design and analysis of new kernel methods, here we focus on how this work can potentially influence the interpretation of deep learning systems. In real-world decision-making problems, interpretability is almost always a crucial factor to consider if one is to deploy a machine learning system. For example, in autonomous driving where NNs are used to detect pedestrians and traffic lights, it is important to understand why this detection network outputs a certain prediction and how confident it is for such a prediction, lack of which can cause damages to the surrounding pedestrians and other drivers. Such a call for interpretability is underlying many works on the ``calibration'' of NNs (see, e.g., \citealt{guo2017calibration}).

Kernel methods, due to its linearity in the feature space, are easier to interpret than highly non-linear NNs, which is typically treated as a black-box. Thus, having a high-quality ``neural-network-simulating'' kernel can greatly simplify the design of ``neural network interpreters'' (like prediction intervals) and may lead to savings of computational resources. However, depending on the user of our technology, there may be negative outcomes. For example, if the user is ignorant of the underlying assumptions behind the validity of our proposed kernels, he/she may have an overt optimism or undue trust on these kernels and make misleading decisions.

We see many potential research directions on improving neural network interpretability by using our kernels. For example, our constructions can be generalized to higher-level truncations of the Hoeffding composition, which may give rise to even better ``neural-network-simulating'' kernels. 
However, to mitigate the risks associated with the question of ``when a kernel is indeed simulating a neural network'', we encourage researchers to carefully examine the validity of the imposed assumptions in a case-by-case manner.

\bibliography{ntkl}
\bibliographystyle{iclr2020_conference}

\appendix
\newpage
\section{Technical Details} 
\label{append:proofs}

\subsection{A Example Validating Claim \ref{claim: limitations_of_label_agnostic_kernels}}\label{subappend:example_limit_of_label_agnostic_kernels}
We first formally define the notion of a ``label system'':
\begin{definition}[Label system]
\label{def: label_system}
Let $(X, Y)$ be a fresh sample from the data distribution. The label system is the function $\eta: \bbR^d \to [0, 1]$ defined by
\[
\eta(\rvx) = \bbP(Y = 1\ | \ X = \rvx).
\]
\end{definition}
Let $\phi(\cdot)$ be the feature map corresponding to $K$, so that $K(\rvx, \rvx') = \la \phi(\rvx), \phi(\rvx') \ra$, where $\la\cdot, \cdot\ra$ is the inner product in the $\phi$-space. By our assumption, $\phi$ is also label-agnostic. 

Suppose $\phi$ works well on the label system $\eta_1$. This means that a simple linear fit in the $\phi$ space suffices to achieve satisfactory accuracy. Geometrically, this translates to the existence of a hyperplane which can almost perfectly separate the two classes. In other words, we can find a normal vector $\rvtheta_1$\footnote{Here we implicitly assume this hyperplane crosses the origin. The construction without this assumption is similar.}, such that $\bbP_{\eta_1}\bigl(Y \cdot \la \rvtheta_1, \phi(X)\ra > 0\bigr) \approx 1$, where we use $\bbP_{\eta_1}$ to stress that $Y$ comes from the label system $\eta_1$. 

Now, choose any vector $\rvtheta_2$ in the $\phi$ space, such that $\rvtheta_2$ is orthogonal to $\rvtheta_1$. We consider the following relabelling procedure, which produces another label system $\eta_2$:
\begin{equation}
	\label{eq: relabelling}
	\eta_2(\rvx) = 
	\begin{cases}
		1 & \textnormal{ if } \langle \rvtheta_1, \phi(\rvx)\rangle \cdot \langle \rvtheta_2, \phi(\rvx)\rangle < 0\\
		0 & \textnormal{ otherwise}.
	\end{cases}
\end{equation}
Under $\eta_2$, the $\phi$-space is partitioned into four quadrants, which we label as $\RN{1}, \RN{2}, \RN{3}, \RN{4}$ counterclockwise. Then, any $\rvx$ in $\RN{1}$ and $\RN{3}$ is labelled as $-1$, and any $\rvx$ in $\RN{2}$ and $\RN{4}$ is labelled as $+1$. Obviously, there is no hyper-plane that can separate the $+1$ and $-1$ examples, meaning that $\phi$ can be arbitrarily bad under $\eta_2$, hence validating Claim. See Fig. \ref{fig:limitation_of_label_agnostic_kernel} for a pictorial illustration.

\begin{figure}[t]
		\centering
		\scalebox{0.6}{
		\includegraphics[width=\textwidth]{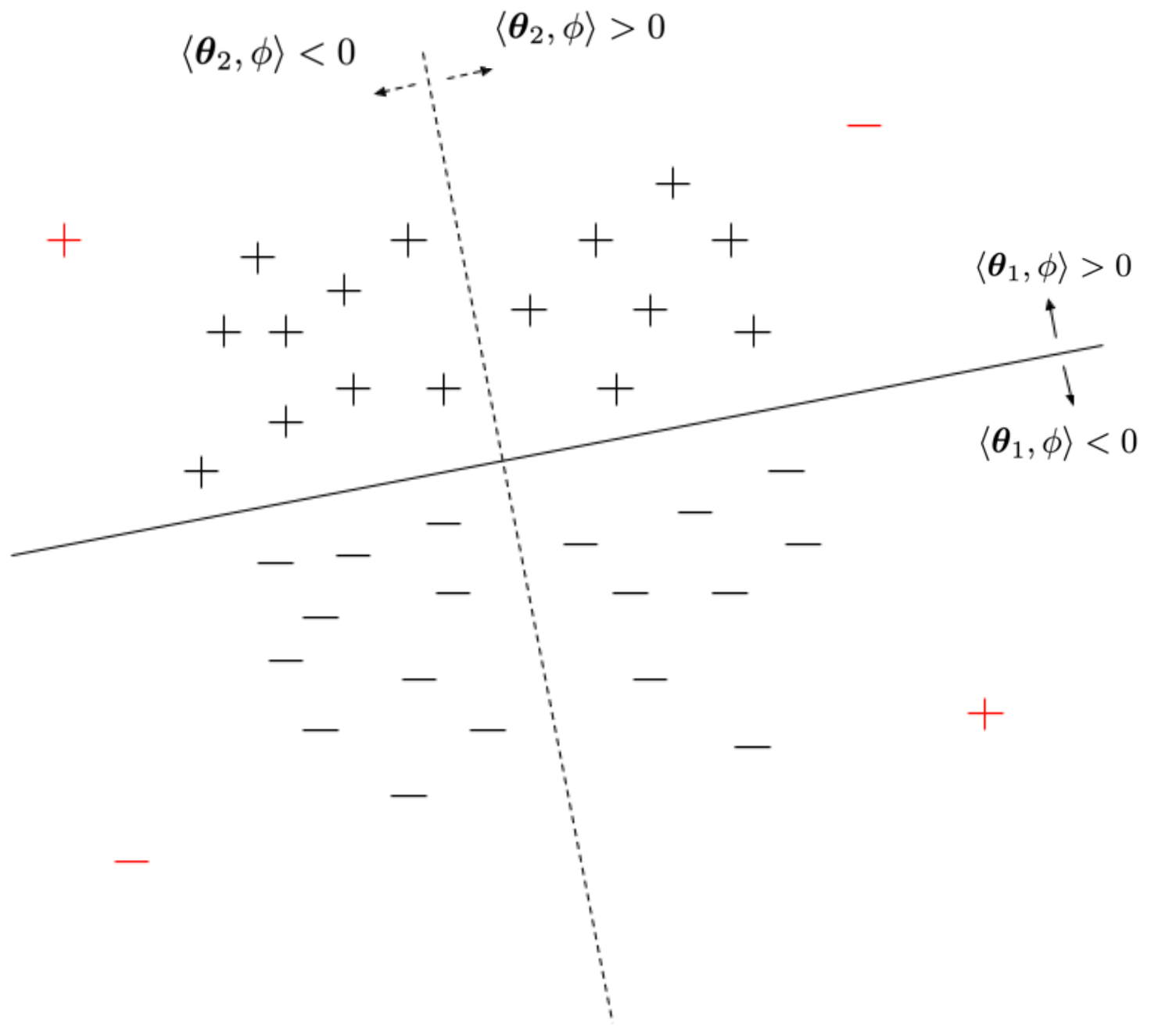}
		}
        \caption{Pictorial illustration of the example that validates Claim \ref{claim: limitations_of_label_agnostic_kernels}. The black ``$+$'' and ``$-$'' represent positive and negative examples according to the label system induced by $\rvtheta_1$, and the solid line represents the corresponding separating hyperplane. The dashed line represents the hyperplane whose normal vector is $\rvtheta_2$, which, along with the solid line, partition that $\phi$ space into four quadrants. The red ``$+$'' and ``$-$'' represents the positive and negative examples under the new label system.}
		\label{fig:limitation_of_label_agnostic_kernel}
\end{figure}

\subsection{Exact integrability of the approximate \texorpdfstring{$\KK_t$}{K2t}} \label{subappend:proof_approx_nth_zero_init}
In this subsection, we prove the following result:
\begin{proposition}[Exact integrability of the approximate $\KK_t$]
	\label{prop: approx_nth_zero_init}
	Assume $\rvf_0 = \mathbf{0}_n$ \footnote{The formula for nonzero $\rvf_0$ can be obtained by similar but lengthier calculations. Since this is not the focus of this paper,  we omit the details.}, then the first three terms in RHS of \eqref{eq: approx_nth_full} is equal to
	\begin{align*}
		 & \KK_0(\rvx, \rvx') + \bigg\langle \KKK_0(\rvx, \rvx', \cdot), (\rmKK_0)^{-1}\bigg(\rmI_n - e^{-t \rmKK_0/n}\bigg) \rvy \bigg\rangle  \\
		 & \qquad + \rvy^\top (\rmKK_0)^{-1}\bigg(\rmI_n - e^{-t \rmKK_0/n}\bigg) \KKKK_0(\rvx, \rvx', \cdot, \cdot) (\rmKK_0)^{-1} \rvy - \rvy^\top \rmP\rmQ(\rvx, \rvx')\rmP^\top \rvy,
	\end{align*}
	where
	\begin{align*}
		\bigg(\rmQ(\rvx, \rvx')\bigg)_{ij} = (1-e^{-t(\rmD_{ii} + \rmD_{jj})/n}) \times {\bigg(\rmP^\top \KKKK_0(\rvx, \rvx', \cdot, \cdot) \rmP \rmD^{-1}\bigg)_{ij}}\bigg/\bigg({\rmD_{ii} + \rmD_{jj}}\bigg)  ,
	\end{align*}
	and $\rmP \rmD \rmP^\top$ is the eigen-decomposition of $\rmKK_0$.
\end{proposition}
\begin{proof}
For notational simplicity, we let $\rmH \equiv \rmKK_0$.
We have
\begin{align*}
	\int_0^t \rvh_u - \rvy du & = \int_0^t e^{-u \rmH/n} \rvy du =  -n\rmH^{-1}\rvy + n\rmH^{-1}e^{-t\rmH/n}\rvy.
\end{align*}
Hence, the second term in the RHS of \eqref{eq: approx_nth_full} is
\begin{align*}
-\frac{1}{n} \int_{0}^t \bigg\langle \KKK_0, \rvh_u-\rvy \bigg\rangle du= \bigg\langle \KKK_0, \rmH^{-1}\bigg(\rmI_n - e^{-t\rmH/n}\bigg) \rvy \bigg\rangle.
\end{align*}
We now deal with the third term in the RHS of \eqref{eq: approx_nth_full}. We have
\begin{align*}
& \int_0^t \int_0^u (\rvh_u - \rvy)^\top \KKKK_0 (\rvh_v - \rvy) dvdu \\
	& = \rvy^\top \int_0^t\int_0^u e^{-u\rmH/n}  \KKKK_0 e^{-v\rmH/n} \rvy dvdu\\
	& = \rvy^\top \int_0^t e^{-u\rmH/n} \KKKK_0 (n\rmH^{-1} - n\rmH^{-1} e^{-u\rmH/n}) \rvy \\
	& = n^2\rvy^\top \rmH^{-1} \bigg(\rmI_n - e^{-t\rmH/n} \bigg) \KKKK_0 \rmH^{-1} \rvy - n\rvy^\top \int_0^t e^{-u\rmH/n} \KKKK_0 \rmH^{-1}e^{-u\rmH/n} \rvy du.
\end{align*}
Let $\rmP \rmD \rmP^\top$ be an eigen-decomposition of $\rmH$. Then we have
\begin{align*}
	\int_0^t e^{-u\rmH/n} \KKKK_0 \rmH^{-1}e^{-u\rmH/n} du& = \int_0^t \rmP e^{-u\rmD/n} \underbrace{\rmP^\top \KKKK_0 \rmP \rmD^{-1}}_{:= \rmM} e^{-u\rmD/n} \rmP^\top du\\
	& = \rmP \int_0^t e^{-t\rmD/n} \rmM e^{-\rmD/n} du \rmP^\top.
\end{align*}
Note that
\[
	\bigg(e^{-t\rmD/n} \rmM e^{-\rmD/n}\bigg)_{ij}  = \rmM_{ij} e^{-u(d_i+d_j)/n},
\]
here $d_i : = \rmD_{ii}$. Hence, 
\[
	\int_0^t \bigg(e^{-t\rmD/n} \rmM e^{-\rmD/n}\bigg)_{ij} du = \rmM_{ij} \cdot \bigg(\frac{n}{d_i + d_j} - \frac{n}{d_i + d_j} e^{-t(d_i+d_j)/n}\bigg) ,
\]
which yields
\[
	\int_0^t e^{-u\rmH/n} \KKKK_0 \rmH^{-1}e^{-u\rmH/n} du = n\rmP  \rmQ \rmP^\top.
\]
Putting the above equations together gives the desired result.
\end{proof}

\subsection{Proof of Lemma \ref{lemma: hoeffding_decomp}}\label{subappend:proof_hoeffding_decomp}
Note that the finite variance assumption in the definition of $\gG_A$ is vacuous, because our label is binary. The rest of the proof is standard (see, e.g., Section 11.4 of \citealt{van2000asymptotic}).

\subsection{Proof of Theorem \ref{thm: approx_nth_error}} \label{subappend:proof_approx_nth_error}

\begin{assump}
	\label{assump: justify_approx_error}
	There exists a small constant $c>0$ such that $c < \|\rvx_i \|_2\leq c^{-1}$ for all $i\in[n]$. Moreover, there exists an integer $p\geq 4$ such that for any $1\leq r\leq 2p+1$, the following two things happen: 
	\begin{enumerate}[itemsep=0mm]
		\item the activation function has a bounded $r$-th derivative; 
		\item there exists a constant $c_r>0$ such that for any distinct indices $1\leq \alpha_1, \alpha_2, \cdots, \alpha_r\leq n$, the smallest singular value of the data matrix $[\rvx_{\alpha_1}, \rvx_{\alpha_2}, \cdots, \rvx_{\alpha_r}]$ is at least $c_r$.
	\end{enumerate}
\end{assump}

To prove Theorem \ref{thm: approx_nth_error}, we first collect some useful details into the following lemma.
\begin{lemma}
	\label{lemma: nth_estimate}
	Under the assumptions of Theorem \ref{thm: approx_nth_error}, with high probability, for any $u\leq t$, we have
	\begin{align*}
		\| \KKK_u\|_\infty & = \tilde{O}\bigg(\frac{1+u}{m}\bigg), \ \ \ \| \KKKK_u\|_\infty = \tilde{O}\bigg(\frac{1}{m}\bigg), \ \ \ \| K_u^{(5)}\|_\infty = \tilde{O}\bigg(\frac{1+u}{m^2}\bigg), \\
		\| \rvf_u -\rvh_u \|_2 & \lesssim \frac{u(1+u)\sqrt{n}}{{m}} \bigg(u \land \frac{n}{\lambda}\bigg), \ \ \ \|\rvf_u - \rvy \|_2 = O(\sqrt{n}).
	\end{align*}
\end{lemma}
\begin{proof}
	This is implied by Equations (C.12), (C.13), (C.16), and (C.28) in \cite{huang2019dynamics}.
\end{proof}

We are now ready to present the proof.
\begin{proof}[Proof of Theorem \ref{thm: approx_nth_error}]
	We have
	\begin{align*}
	& \KK_t(\rvx_\alpha, \rvx_\beta)\\
	& = \KK_0(\rvx_\alpha, \rvx_\beta) - \frac{1}{n} \bigg\langle \KKK_0,\int_0^t \rvh_u - \rvy + \rvf_u - \rvh_u du  \bigg\rangle \\
	& \ \ \ + \frac{1}{n^2} \int_0^t \int_0^u (\rvh_u - \rvy + \rvf_u - \rvh_u)^\top (\KKKK_0 + \KKKK_v - \KKKK_0) (\rvh_v - \rvy + \rvf_v - \rvh_v) dvdu\\
	& = \KK_0(\rvx_\alpha, \rvx_\beta) - \frac{1}{n}  \bigg\langle \KKK_0,\int_0^t \rvh_u - \rvy du \bigg\rangle  + \frac{1}{n^2}  \int_0^t\int_0^u (\rvh_u - y)^\top \KKKK_0(\rvh_v-y) dvdu + \gE,
	\end{align*}
	where the error term $\gE$ is given by
	\begin{align*}
		\gE & = - \frac{1}{n}  \bigg\langle \KKK_0,\int_0^t \rvf_u - \rvh_u du  \bigg\rangle + \frac{1}{n^2}\int_0^t\int_0^u (\rvh_u-\rvy)^\top \KKKK_0(\rvf_v - \rvh_v) dvdu\\
	& \qquad  + \frac{1}{n^2} \int_0^t\int_0^u (\rvh_u-\rvy)^\top (\KKKK_v - \KKKK_0)(\rvf_v - \rvy) dvdu \\
	& \qquad + \frac{1}{n^2} \int_0^t\int_0^u (\rvf_u - \rvh_u)^\top \KKKK_v(\rvf_v - \rvy) dvdu.
	\end{align*}
	We label the four terms in the RHS above as $\RN{1}, \RN{2}, \RN{3}$, and $\RN{4}$. By lemma \ref{lemma: nth_estimate}, for the first term, w.h.p. we have
	\begin{align*}
		\RN{1} & \leq \frac{1}{n} \| \KKK_0(\rvx_\alpha,\rvx_\beta, \cdot)\|_2 \cdot \int_0^t  \| \rvf_u-\rvh_u\|_2 du \\ 
		& \lesssim \frac{1}{{m}} \| \KKK_0\|_\infty \cdot \int_0^t u(1+u) \bigg(u\land \frac{n}{\lambda}\bigg) du\\
		& \lesssim \frac{(\log m)^c}{m^2} \cdot \bigg((t^3 + t^4) \land  \frac{n (t^2 + t^3)}{\lambda}\bigg)\\
		& \lesssim \frac{(\log m)^c \cdot t^2(1+t)}{m^2} \bigg(t \land \frac{n}{\lambda}\bigg).
	\end{align*}
	For the second term, we have
	\begin{align*}
		\RN{2} & \leq \frac{1}{n^2} \int_0^t \int_0^u \| \rvh_u - \rvy\|_2 \| \rvf_v -\rvh_v\|_2 \| \KKKK_0(\rvx_\alpha, \rvx_\beta, \cdot, \cdot)\| dvdu\\
		& \lesssim\frac{1}{m} \int_0^t \int_0^u v(1+v) \bigg(v \land \frac{n}{\lambda}\bigg) \| \KKKK_0\|_\infty dvdu\\
		& \lesssim \frac{(\log m)^c}{m^2} \int_0^t \int_0^u v(1+v) \bigg(v \land \frac{n}{\lambda}\bigg)  \\
		& \lesssim \frac{(\log m)^c}{m^2} (t^4 +t^5) \land \frac{n (t^3 + t^4)}{\lambda}\\
		& = \frac{(\log m)^c\cdot t^3(1+t)}{m^2} \bigg(t \land \frac{n}{\lambda}\bigg) .
	\end{align*}
	For the third term, we have
	\begin{align*}
		\RN{3} & \leq \frac{1}{n^2} \int_0^t \int_0^u \|\rvh_u - \rvy \|_2 \| \rvf_v-\rvy\|_2 \|\KKKK_v(\rvx_\alpha, \rvx_\beta, \cdot, \cdot) - \KKKK_0(\rvx_\alpha, \rvx_\beta, \cdot, \cdot) \|_2 dvdu.
	\end{align*}
	Note that 
	\begin{align*}
		\| \rvh_u -\rvy\|_2&  \leq \|\rvh_u -\rvf_u \|_2 + \| \rvf_u -\rvy\|_2 \lesssim \frac{u(1+u)\sqrt{n}}{m}\bigg(u \land \frac{n}{\lambda}\bigg) + \sqrt{n},\\
		\| \rvf_v- \rvy\|_2 & \lesssim \sqrt{n},
	\end{align*}	
	and
	\begin{align*}
		& \|\KKKK_v(\rvx_\alpha, \rvx_\beta, \cdot, \cdot) - \KKKK_0(\rvx_\alpha, \rvx_\beta, \cdot, \cdot) \|_2\\	
		& \leq n \max_{i, j\in[n]}|\KKKK_v(\rvx_\alpha, \rvx_\beta, \rvx_i, \rvx_j) - \KKKK_0(\rvx_\alpha, \rvx_\beta, \rvx_i, \rvx_j)|\\
		& \leq n \int_0^v \max_{i, j\in[n]}| \dot{K}_w^{(4)}(\rvx_\alpha, \rvx_\beta, \rvx_i, \rvx_j)| dw\\
		& = n\int_0^v \max_{i, j\in[n]} \frac{1}{n} \bigg|\sum_{k\in[n]} K^{(5)}_w(\rvx_\alpha, \rvx_\beta, \rvx_i, \rvx_j, \rvx_k) (f_w(\rvx_k) - \rvy_k)\bigg| dw\\
		& \leq n \int_0^v \| K^{(5)}_w\|_\infty \frac{1}{\sqrt{n}}  \| \rvf_w - \rvy\|_2 dw\\
		& \lesssim \frac{n(\log m)^c}{m^2}\int_0^v  (1+w) dw\\
		& \lesssim \frac{n(\log m)^c\cdot (v+ v^2)}{m^2} .
	\end{align*}
	Hence, we can bound the third term by
	\begin{align*}
		\RN{3} & \lesssim \frac{(\log m)^c}{m^2} \int_0^t \int_0^u (v+v^2) \bigg(1 + \frac{u(1+u)(u \land n/\lambda)}{m}\bigg) dvdu\\
		& \lesssim  \frac{(\log m)^c \cdot t^3(1+t)}{m^2} +  \frac{(\log m)^c \cdot t^4(1+t+t^2)}{m^3}  \bigg(t \land \frac{n}{\lambda}\bigg).
	\end{align*}
	Finally, we bound the fourth term by
	\begin{align*}
		\RN{4} & \leq \frac{1}{n^2} \int_0^t\int_0^u \| \rvf_u-\rvh_u\|_2 \|\rvf_v-\rvy \|_2 \| \KKKK_v(\rvx_\alpha,\rvx_\beta,\cdot,\cdot)\|_2 dvdu\\
		& \lesssim \frac{(\log m)^c}{m^2}  \int_0^t\int_0^u u(1+u) \bigg(u\land \frac{n}{\lambda}\bigg) dvdu\\
		& \lesssim \frac{(\log m)^c \cdot t^3(1+t)}{m^2} \bigg(t \land \frac{n}{\lambda}\bigg).
	\end{align*}
	Combining the above four bounds gives the desired bound on $\gE$. The bounds for the three terms in the RHS of \eqref{eq: approx_nth_full} are derived similarly, and we omit the details.
\end{proof}

\section{NTH-Related Calculations}\label{append:nth_calculations}

\subsection{An Symbolic Program to Compute NTH}
In this section, we develop a recursive program to symbolically compute $K^{(r)}_t$, the $r$-th order kernel in NTH. For simplicity, we consider neural nets without biases \footnote{The method developed in this section applies to neural nets with biases, but with lengthier calculations.}:
\[
f(\rvx, \rvtheta) = \WWW{L} \sigma\bigg(\WWW{L-1} \cdots \bigl(\WWW{2} \sigma(\WWW{1} \rvx ) \bigr) \bigg).
\]

We begin by noting that $K^{(2)}_t$ can be written as an inner product of two gradients. We will see that $K^{(3)}_t$ can be written as a quadratic form, and $K^{(4)}$ can be written as a cubic form, etc. 

To this end, let us denote $\nabla f(\rvx, \rvtheta)$ to be the partial derivative of $f(\rvx, \rvtheta)$ w.r.t. $\rvtheta$. We sometimes drop the dependence on $t$ and write $f(\rvx_\alpha, \rvtheta_t) \equiv f_\alpha$ when there is no ambiguity. 

We will regard $\rvtheta = \rmW  = \{\rmW ^{(\ell)}_{j k }: \ell\in[L], j\in[p_\ell], j \in [q_\ell]\}$ as a rank-$3$ tensor. 
We write $\WWW{\ell}_{jk}\equiv \rmW _{\ell jk}$ when there is no ambiguity. 
In our current notations, $\nabla f(\rvx, \rvtheta) = \{\bigl(\nabla f(\rvx, \rvtheta)\bigr)_{\ell j k}: \ell\in[L], j\in p_\ell, k\in q_\ell\}$ is also a rank-$3$ tensor. We begin by writing
\begin{align*}
	\KK_t(\rvx_\alpha, \rvx_\beta) & = \la \nabla f(\rvx_\alpha, \rvtheta_t) , \nabla f(\rvx_\beta, \rvtheta_t)\ra \\
	& = \sum_{\ell\in[L]} \sum_{j\in[p_\ell]} \sum_{k\in[q_\ell]} \frac{\partial f(\rvx_\alpha, \rvtheta_t)}{\partial \WWW{\ell}_{j k}} \cdot \frac{\partial f(\rvx_\beta, \rvtheta_t)}{\partial \WWW{\ell}_{j k}} \\
	& = (\nabla f_\alpha)^{\ell j k } (\nabla f_\beta)_{\ell j k},
\end{align*}
where in the last line we have used the \emph{Einstein notation} \footnote{That is, if an index appears twice, we take the sum over this index.}. Taking derivative w.r.t. $t$ gives
\begin{align*}
	\frac{d}{dt}\KK_t(\rvx_\alpha, \rvx_\beta) & =  (\frac{d}{dt} \nabla f_\alpha)^{\ell j k } (\nabla f_\beta)_{\ell j k} + (\nabla f_\alpha)^{\ell j k } (\frac{d}{dt} \nabla f_\beta)_{\ell j k}.
\end{align*}
We have
\begin{align*}
	 (\frac{d}{dt} \nabla f_\alpha)_{\ell j k } & = \frac{\partial (\nabla f_\alpha)_{\ell j k }}{\partial \rmW _{suv}}  \cdot \frac{d \rmW _{suv}}{dt} \\
	 & = \frac{\partial (\nabla f_\alpha)_{\ell j k }}{\partial \rmW _{suv}}  \cdot 
	 {\frac{-d L(\rvtheta_t)}{d \rmW _{suv}}} \\
	 & = \frac{\partial (\nabla f_\alpha)_{\ell j k }}{\partial \rmW _{suv}} \cdot \bigg(-\frac{1}{n} \sum_{\gamma\in[n]} (f_\gamma - \rvy_\gamma) \cdot (\nabla f_\gamma)_{suv} \bigg) \\
	 & = -\frac{1}{n} \sum_{\gamma\in[n]} (f_\gamma- \rvy_\gamma) \cdot  \frac{\partial (\nabla f_\alpha)_{\ell j k }}{\partial \rmW _{suv}} \cdot (\nabla f_\gamma)_{suv} \\
	 &{=} -\frac{1}{n} \sum_{\gamma\in[n]} (f_\gamma- \rvy_\gamma) \cdot (\nabla^2 f_\alpha)^{\ell j k}_{suv} \cdot (\nabla f_\gamma)_{suv},
\end{align*}
where we denote $\nabla^2 f$ to be the rank-$6$ tensor, whose $(\ell, j, k, s, u, v)$-th entry is given by 
\[
{\partial}(\nabla f)_{\ell j k } / \partial \rmW _{suv} = \frac{\partial^2 f}{\partial \rmW _{suv} \rmW _{\ell j k}}.
\] 
A similar computation gives
\[
	(\frac{d}{dt} \nabla f_\beta)_{\ell j k } = -\frac{1}{n} \sum_{\gamma\in[n]} (f_\gamma- \rvy_\gamma) \cdot (\nabla^2 f_\beta)^{\ell j k}_{suv} \cdot (\nabla f_\gamma)_{suv}.
\]
Hence, we arrive at
\begin{align*}
	& \KKK_t(\rvx_\alpha,\rvx_\beta,\rvx_\gamma) \nonumber\\
	& = (\nabla^2 f_\alpha)^{\ell j k }_{s u v} (\nabla f_\beta)_{\ell j k } (\nabla f_\gamma)_{suv}  + (\nabla f_\alpha)_{\ell j k } (\nabla^2 f_\beta)^{\ell j k }_{s u v} (\nabla f_\gamma)_{s u v} \\
	& = \sum_{\ell \in [L]} \sum_{j \in[p_\ell]} \sum_{k\in[q_\ell]} \sum_{s \in [L]} \sum_{u\in [p_s]} \sum_{v \in [q_s]} \frac{\partial f(\rvx_\beta, \rvtheta_t)}{\partial \rmW _{\ell j k }} \cdot \frac{\partial^2 f(\rvx_\alpha, \rvtheta_t)}{\partial \rmW _{\ell j k} \partial \rmW _{s u v}} \cdot \frac{\partial f(\rvx_\gamma, \rvtheta_t)}{ {\partial} \rmW _{suv}} \\
	& \qquad + \frac{\partial f(\rvx_\alpha, \rvtheta_t)}{\partial \rmW _{\ell j k }} \cdot \frac{\partial^2 f(\rvx_\beta, \rvtheta_t)}{\partial \rmW _{\ell j k} \partial \rmW _{s u v}} \cdot \frac{\partial f(\rvx_\gamma, \rvtheta_t)}{{\partial} \rmW _{suv}} \nonumber.
\end{align*}	

Note that the above expression is a quadratic form:
\[
	\KKK_t(\rvx_\alpha,\rvx_\beta,\rvx_\gamma) = (\nabla f_{\beta})^\top (\nabla^2 f_\alpha) (\nabla f_\gamma) + (\nabla f_{\alpha})^\top (\nabla^2 f_\beta) (\nabla f_\gamma).
\]

We have already seen some patterns showing up. To obtain $\KKK_t$ from $\KK_t$, we simply conduct the following program:
\begin{enumerate}[itemsep=0mm]
	\item Start with $\KK_t = (\nabla f_\alpha)^{\ell j k } (\nabla f_\beta)_{\ell j k}$ in Einstein notation, which is a function of $r = 2$ gradients;
	\item Introduce a new index $\gamma$ for data points, and a new set of indices $(s u v)$ for weights;
	\item Replicate $\KK_t$ for $r=2$ times, and append $(\nabla f_\gamma)_{suv}$ to the end of each term: 
	\[
		(\nabla f_\alpha)^{\ell j k } (\nabla f_\beta)_{\ell j k} (\nabla f_\gamma)_{suv}  + (\nabla f_\alpha)^{\ell j k } (\nabla f_\beta)_{\ell j k} (\nabla f_\gamma)_{suv} ;
	\]
	\item Choose a term from a total of $r=2$ terms in $\KK_t$, raise its gradient to one higher level, add the new indices $(suv)$ to this term, and do this operation in all possible ways:
	\[
		\KKK_t = (\nabla^2 f_\alpha)^{\ell j k }_{suv} (\nabla f_\beta)_{\ell j k} (\nabla f_\gamma)_{suv}  + (\nabla f_\alpha)_{\ell j k } (\nabla^2 f_\beta)^{\ell j k}_{s u v} (\nabla f_\gamma)_{suv}.
	\]
\end{enumerate}
We now apply the above program to obtain $\KKKK_t$ from $\KKK_t$:
\begin{enumerate}
	\item Introduce a new index $\xi$ for data points, and a new set of indices $(abc)$;
	\item Since $\KKK$ is a function of $r = 3$ gradients, we replicate $\KKK_t$ for $3$ times, and append $(\nabla f_\xi)_{abc}$ to the end of each term:
	\begin{align*}
		& (\nabla^2 f_\alpha)^{\ell j k }_{suv} (\nabla f_\beta)_{\ell j k} (\nabla f_\gamma)_{suv} (\nabla f_\xi)_{abc} \\
		& + (\nabla^2 f_\alpha)^{\ell j k }_{suv} (\nabla f_\beta)_{\ell j k} (\nabla f_\gamma)_{suv} (\nabla f_\xi)_{abc} \\
		& + (\nabla^2 f_\alpha)^{\ell j k }_{suv} (\nabla f_\beta)_{\ell j k} (\nabla f_\gamma)_{suv} (\nabla f_\xi)_{abc}  \\
		& +  (\nabla f_\alpha)_{\ell j k } (\nabla^2 f_\beta)^{\ell j k}_{s u v} (\nabla f_\gamma)_{suv}  (\nabla f_\xi)_{abc}\\
		& +  (\nabla f_\alpha)_{\ell j k } (\nabla^2 f_\beta)^{\ell j k}_{s u v} (\nabla f_\gamma)_{suv} (\nabla f_\xi)_{abc}\\
		& +  (\nabla f_\alpha)_{\ell j k } (\nabla^2 f_\beta)^{\ell j k}_{s u v} (\nabla f_\gamma)_{suv}  (\nabla f_\xi)_{abc};
	\end{align*}	
	\item Raise a term's gradient (from the former $3$ terms) to a higher level, add the new indices, and do it in all possible ways:
	\begin{align}
		\KKKK_t = &   (\nabla^3 f_\alpha)^{\ell j k }_{suv, abc} (\nabla f_\beta)_{\ell j k} (\nabla f_\gamma)_{suv} (\nabla f_\xi)_{abc}\nonumber \\
		& + (\nabla^2 f_\alpha)^{\ell j k }_{suv} (\nabla^2 f_\beta)^{\ell j k}_{abc} (\nabla f_\gamma)_{suv} (\nabla f_\xi)_{abc} \nonumber\\
		& + (\nabla^2 f_\alpha)^{\ell j k }_{suv} (\nabla f_\beta)_{\ell j k} (\nabla^2 f_\gamma)^{suv}_{abc} (\nabla f_\xi)_{abc} \nonumber \\
		& +  (\nabla^2 f_\alpha)^{\ell j k }_{abc} (\nabla^2 f_\beta)^{\ell j k}_{s u v} (\nabla f_\gamma)_{suv}  (\nabla f_\xi)_{abc} \nonumber\\
		& +  (\nabla f_\alpha)_{\ell j k } (\nabla^3 f_\beta)^{\ell j k}_{s u v, abc} (\nabla f_\gamma)_{suv} (\nabla f_\xi)_{abc} \nonumber\\
		\label{eq: K_4_general}
		& +  (\nabla f_\alpha)_{\ell j k } (\nabla^2 f_\beta)^{\ell j k}_{s u v} (\nabla^2 f_\gamma)^{suv}_{abc}  (\nabla f_\xi)_{abc}.
	\end{align}
\end{enumerate}
In the above program, we have regarded $\nabla^3 f$ as a rank-$9$ tensor, whose $(\ell, j, k, s, u, v, a, b, c)$-th entry is given by
\[
	\frac{\partial^3 f}{\partial \rmW _{abc} {\partial} \rmW _{suv} {\partial} \rmW _{\ell j k} }.
\]

The correctness of the above recursive symbolic program can be proved by straightforward induction, and we omit the details.

\subsection{Explicit Expressions for NTH for Two-Layer Nets}
Taking advantage of the recursive program, it's relatively easy to get explicit expressions for $K^{(r)}$ when $r$ is not too large. We focus on the following two-layer net
\begin{equation}
	f(\rvx, \rvtheta) = \frac{1}{\sqrt{m}} \rva^\top \sigma(\rmW  \rvx).
\end{equation}

\begin{proposition}[Expression for $\KK_t, \KKK_t$ and $\KKKK_t$ in a two-layer net]
\label{prop: K_2_K_3_two_layer_net}
	For the two layer neural network $f(\rvx, \rvtheta) = \frac{1}{\sqrt{m}}\rva^\top \sigma(\rmW  \rvx)$, where $\rva\in \bbR^m, \rmW \in \bbR^{m\times d}$, and $\sigma$ an activation function, we have
	\begingroup
	\allowdisplaybreaks
	\begin{align}
		& \KK_t(\rvx_\alpha, \rvx_\beta) \nonumber\\
		\label{eq: K_2_two_layer}	
		&= \frac{1}{m} \rvx_\alpha^\top  \rvx_\beta \cdot \bigg\la \sigma'(\rmW  \rvx_\alpha)\odot \rva_t, \sigma'(\rmW  \rvx_\beta)\odot \rva_t  \bigg\ra + \frac{1}{m} \bigg\la \sigma(\rmW  \rvx_\alpha), \sigma(\rmW  \rvx_\beta) \bigg\ra ,\\
		& \KKK_t(\rvx_\alpha, \rvx_\beta, \rvx_\gamma) \nonumber\\
		& =\frac{1}{m\sqrt{m}} \rvx_\alpha^\top \rvx_\gamma \cdot  \rvx_\alpha^\top \rvx_\beta \cdot \bigg\la \rva_t \odot \rva_t\odot \rva_t , \sigma''(\rmW  \rvx_\alpha)\odot \sigma'(\rmW  \rvx_\beta)\odot \sigma'(\rmW  \rvx_\gamma) \bigg\ra \nonumber \\
		& \qquad + \frac{1}{m\sqrt{m}}\rvx_\beta^\top\rvx_\gamma \cdot \rvx_\alpha^\top\rvx_\beta \cdot \bigg\la\rva_t \odot\rva_t\odot\rva_t , \sigma'(\rmW\rvx_\alpha)\odot \sigma''(\rmW\rvx_\beta) \odot \sigma'(\rmW\rvx_\gamma) \bigg\ra \nonumber \\
		& \qquad + \frac{2}{m\sqrt{m}} \rvx_\alpha^\top\rvx_\beta \cdot \bigg\la\rva_t, \sigma'(\rmW\rvx_\alpha)\odot \sigma'(\rmW\rvx_\beta)\odot \sigma(\rmW\rvx_\gamma)\bigg\ra \nonumber\\ 
		& \qquad + \frac{1}{m\sqrt{m}}\rvx_\alpha^\top\rvx_\gamma \cdot \bigg\la\rva_t, \sigma'(\rmW \rvx_\alpha)\odot \sigma(\rmW \rvx_\beta)\odot \sigma'(\rmW \rvx_\gamma) \bigg\ra \nonumber \\
		\label{eq: K_3_two_layer}
		& \qquad + \frac{1}{m\sqrt{m}}\rvx_\beta^\top\rvx_\gamma \cdot \bigg\la\rva_t, \sigma(\rmW \rvx_\alpha)\odot \sigma'(\rmW \rvx_\beta)\odot \sigma'(\rmW \rvx_\gamma) \bigg\ra ,\\
		& \KKKK_t(\rvx_\alpha,\rvx_\beta,\rvx_\gamma,\rvx_\xi)\nonumber\\
		& =  \frac{1}{m^2}\rvx_\alpha^\top\rvx_\beta \cdot \rvx_\alpha^\top\rvx_\gamma \cdot\rvx_\alpha^\top\rvx_\xi \cdot \bigg\la\rva_t^{\odot 4}, \sigma'''(\rmW \rvx_\alpha) \odot \sigma'(\rmW \rvx_\beta) \odot \sigma'(\rmW  \rvx_\gamma) \odot \sigma'(\rmW \rvx_\xi) \bigg\ra \nonumber\\
		& + \frac{1}{m^2}\rvx_\beta^\top\rvx_\alpha \cdot\rvx_\beta^\top\rvx_\gamma \cdot\rvx_\beta^\top\rvx_\xi \cdot \bigg\la\rva_t^{\odot 4}, \sigma'(\rmW \rvx_\alpha) \odot \sigma'''(\rmW \rvx_\beta) \odot \sigma'(\rmW \rvx_\gamma) \odot \sigma'(\rmW \rvx_\xi) \bigg\ra \nonumber \\
		& + \frac{1}{m^2}\rvx_\alpha^\top\rvx_\beta \cdot (\rvx_\alpha^\top\rvx_\gamma \cdot\rvx_\beta^\top\rvx_\xi +\rvx_\alpha^\top\rvx_\xi \cdot\rvx_\beta^\top\rvx_\gamma) \cdot \bigg\la\rva^{\odot 4}, \sigma''(\rmW \rvx_\alpha) \odot \sigma''(\rmW \rvx_\beta) \odot \sigma'(\rmW \rvx_\gamma) \odot \sigma'(\rmW \rvx_\xi)\bigg\ra \nonumber \\
		& + \frac{1}{m^2}\rvx_\alpha^\top\rvx_\gamma \cdot\rvx_\alpha^\top\rvx_\beta \cdot\rvx_\gamma^\top\rvx_\xi \cdot \bigg\la\rva^{\odot 4} , \sigma''(\rmW \rvx_\alpha) \odot \sigma'(\rmW \rvx_\beta) \odot \sigma''(\rmW \rvx_\gamma) \odot \sigma'(\rmW \rvx_\xi) \bigg\ra \nonumber \\
		& + \frac{1}{m^2}\rvx_\beta^\top\rvx_\gamma \cdot\rvx_\beta^\top\rvx_\alpha  \cdot\rvx_\gamma^\top\rvx_\xi \cdot \bigg\la\rva^{\odot 4}, \sigma'(\rmW \rvx_\alpha) \odot \sigma''(\rmW \rvx_\beta) \odot \sigma''(\rmW \rvx_\gamma) \odot \sigma'(\rmW \rvx_\xi) \bigg\ra \nonumber \\
		& + \frac{3}{m^2}\rvx_\alpha^\top\rvx_\beta \cdot\rvx_\alpha^\top\rvx_\gamma \cdot \bigg\la\rva^{\odot 2},  \sigma''(\rmW \rvx_\alpha) \odot \sigma'(\rmW \rvx_\beta) \odot \sigma'(\rmW \rvx_\gamma) \odot \sigma(\rmW \rvx_\xi)  \bigg\ra \nonumber \\
		& + \frac{3}{m^2}\rvx_\beta^\top\rvx_\alpha \cdot\rvx_\beta^\top\rvx_\gamma \cdot \bigg\la\rva^{\odot 2},  \sigma'(\rmW \rvx_\alpha) \odot \sigma''(\rmW \rvx_\beta) \odot \sigma'(\rmW \rvx_\gamma) \odot \sigma(\rmW \rvx_\xi)  \bigg\ra \nonumber \\
		& + \frac{2}{m^2}\rvx_\alpha^\top\rvx_\beta \cdot\rvx_\alpha^\top\rvx_\xi \cdot \bigg\la\rva^{\odot 2},\sigma''(\rmW \rvx_\alpha) \odot \sigma'(\rmW \rvx_\beta) \odot \sigma(\rmW \rvx_\gamma) \odot \sigma'(\rmW \rvx_\xi)   \bigg\ra \nonumber \\
		& + \frac{2}{m^2}\rvx_\beta^\top\rvx_\alpha\cdot\rvx_\beta^\top\rvx_\xi \cdot \bigg\la\rva^{\odot 2},\sigma'(\rmW \rvx_\alpha) \odot \sigma''(\rmW \rvx_\beta) \odot \sigma(\rmW \rvx_\gamma) \odot \sigma'(\rmW \rvx_\xi)   \bigg\ra \nonumber \\
		& + \frac{2}{m^2}\rvx_\alpha^\top\rvx_\beta\cdot\rvx_\gamma^\top\rvx_\xi \cdot \bigg\la\rva^{\odot 2},\sigma'(\rmW \rvx_\alpha) \odot \sigma'(\rmW \rvx_\beta) \odot \sigma'(\rmW \rvx_\gamma) \odot \sigma'(\rmW \rvx_\xi)   \bigg\ra \nonumber \\
		& + \frac{1}{m^2}\rvx_\alpha^\top\rvx_\gamma\cdot\rvx_\alpha^\top\rvx_\xi \cdot \bigg\la\rva^{\odot 2},\sigma''(\rmW \rvx_\alpha) \odot \sigma(\rmW \rvx_\beta) \odot \sigma'(\rmW \rvx_\gamma) \odot \sigma'(\rmW \rvx_\xi)   \bigg\ra \nonumber \\
		& + \frac{1}{m^2}\rvx_\gamma^\top\rvx_\alpha\cdot\rvx_\gamma^\top\rvx_\xi \cdot \bigg\la\rva^{\odot 2},\sigma'(\rmW \rvx_\alpha) \odot \sigma(\rmW \rvx_\beta) \odot \sigma''(\rmW \rvx_\gamma) \odot \sigma'(\rmW \rvx_\xi)   \bigg\ra \nonumber \\
		& + \frac{1}{m^2}\rvx_\beta^\top\rvx_\gamma\cdot\rvx_\beta^\top\rvx_\xi \cdot \bigg\la\rva^{\odot 2},\sigma(\rmW \rvx_\alpha) \odot \sigma''(\rmW \rvx_\beta) \odot \sigma'(\rmW \rvx_\gamma) \odot \sigma'(\rmW \rvx_\xi)   \bigg\ra \nonumber \\
		& + \frac{1}{m^2}\rvx_\gamma^\top\rvx_\beta\cdot\rvx_\gamma^\top\rvx_\xi \cdot \bigg\la\rva^{\odot 2},\sigma(\rmW \rvx_\alpha) \odot \sigma'(\rmW \rvx_\beta) \odot \sigma''(\rmW \rvx_\gamma) \odot \sigma'(\rmW \rvx_\xi)   \bigg\ra \nonumber \\
		& + \frac{1}{m^2} (\rvx_\alpha^\top\rvx_\gamma\cdot\rvx_\beta^\top\rvx_\xi +\rvx_\alpha^\top\rvx_\xi\cdot\rvx_\beta^\top\rvx_\gamma) \cdot \bigg\la\rva^{\odot 2},\sigma'(\rmW \rvx_\alpha) \odot \sigma'(\rmW \rvx_\beta) \odot \sigma'(\rmW \rvx_\gamma) \odot \sigma'(\rmW \rvx_\xi)   \bigg\ra \nonumber \\
		& + \frac{2}{m^2}\rvx_\alpha^\top\rvx_\beta \cdot \bigg\la \mathbf{1}_m, \sigma'(\rmW \rvx_\alpha) \odot \sigma'(\rmW \rvx_\beta) \odot \sigma(\rmW \rvx_\gamma) \odot \sigma(\rmW \rvx_\xi)  \bigg\ra \nonumber \\
		& + \frac{1}{m^2}\rvx_\alpha^\top\rvx_\gamma \cdot \bigg\la \mathbf{1}_m, \sigma'(\rmW \rvx_\alpha) \odot \sigma(\rmW \rvx_\beta) \odot \sigma'(\rmW \rvx_\gamma) \odot \sigma(\rmW \rvx_\xi)  \bigg\ra \nonumber \\
		\label{eq: K_4_two_layer}	
		& + \frac{1}{m^2}\rvx_\beta^\top\rvx_\gamma \cdot \bigg\la \mathbf{1}_m, \sigma(\rmW \rvx_\alpha) \odot \sigma'(\rmW \rvx_\beta) \odot \sigma'(\rmW \rvx_\gamma) \odot \sigma(\rmW \rvx_\xi)  \bigg\ra , 
	\end{align}
	\endgroup
	where we let $\rva^{\odot r}$ to be the vector whose $i$-th entry is $\rva_i^r$.
\end{proposition}
\begin{proof}
	{\bf Computation of $\KK_t$.} We first compute $\KK_t(\rvx_\alpha,\rvx_\beta) = (\nabla f_\alpha)_{\ell j k } (\nabla f_\beta)_{\ell j k}$. We have
	\begin{align*}
		\frac{\partial f(\rvx , \rvtheta)}{\partial \rva } & =  \frac{1}{\sqrt{m}} \sigma(\rmW \rvx) \\
		\frac{\partial f(\rvx , \rvtheta)}{\partial \rmW} & = \frac{1}{\sqrt{m}} \diag(\sigma'(\rmW \rvx)  \rva \rvx^\top = \frac{1}{\sqrt{m}} \bigg(\sigma'(\rmW \rvx \odot \rva )\bigg) x^\top.
	\end{align*}
	Then Equation \eqref{eq: K_2_two_layer} follows by trivial algebra. 

	{\bf Computation of $\KKK_t$.} Now consider $\KKK_t$. We have
	\begin{align*}
		\KKK_t(\rvx_\alpha,\rvx_\beta,\rvx_\gamma) = (\nabla^2 f_\alpha)^{\ell j k }_{suv} (\nabla f_\beta)_{\ell j k} (\nabla f_\gamma)_{suv}  + (\nabla f_\alpha)_{\ell j k } (\nabla^2 f_\beta)^{\ell j k}_{s u v} (\nabla f_\gamma)_{suv}.
	\end{align*}
	For the $\ell = s = 2$ term, we have $(\nabla^2 f)^{2 j k}_{2 u v} = 0$, since $\partial f/ \partial \rva $ is constant in $\rva$. For the $\ell = s = 1$ term, we have
	\begin{align*}
		(\nabla^2 f)^{1 j k}_{1 u v} & = \frac{\partial^2 f(\rvx , \rvtheta)}{\partial \rmW_{uv} W_{jk}} \\
		& = \frac{\partial }{\partial \rmW_{uv}} \frac{1}{\sqrt{m}} \sigma'(\rmW \rvx_j) \rva_j\rvx_k \\
		& = \frac{1}{\sqrt{m}} \sigma''(\rmW \rvx_j) \delta_{ju}\rvx_v \rva_j\rvx_k,
	\end{align*}
	where $\delta_{ju}$ is the Kronecker delta function. Hence we have
	\begin{align*}
		(\nabla^2 f_\alpha)^{1 j k }_{1uv} (\nabla f_\beta)_{1 j k} (\nabla f_\gamma)_{1uv} & =  \frac{1}{m\sqrt{m}} \sigma''(\rmW \rvx_\alpha)_j \sigma'(\rmW \rvx_\beta)_j \sigma'(\rmW \rvx_\gamma)_j \rva_j^3\rvx_{\alpha, v}\rvx_{\gamma, v}\rvx_{\alpha, k}\rvx_{\beta, k} \\
		& = \frac{1}{m\sqrt{m}}\rvx_\alpha^\top\rvx_\gamma \cdot \rvx_\alpha^\top\rvx_\beta \cdot \bigg\la \rva \odot  \rva\odot \rva , \sigma''(\rmW \rvx_\alpha)\odot \sigma'(\rmW \rvx_\beta)\odot \sigma'(\rmW \rvx_\gamma) \bigg\ra.
	\end{align*}
	A similar computation gives
	\[
		(\nabla f_\alpha)_{1 j k } (\nabla^2 f_\beta)^{1 j k}_{1uv} (\nabla f_\gamma)_{1uv}
		= \frac{1}{m\sqrt{m}}\rvx_\beta^\top\rvx_\gamma \cdot \rvx_\alpha^\top\rvx_\beta \cdot \bigg\la \rva \odot  \rva\odot \rva  , \sigma'(\rmW \rvx_\alpha)\odot \sigma''(\rmW \rvx_\beta) \odot \sigma'(\rmW \rvx_\gamma) \bigg\ra.
	\]
	For the $\ell=1, s = 2$ term, note that $a$ is a vector (also a row matrix), so $u = 1$. This gives 
	\begin{align*}
		(\nabla^2 f_\alpha)^{1jk}_{21v} (\nabla f_\beta)_{1jk} (\nabla f_\gamma)_{21 v} & = \frac{\partial}{\partial \rva_v}\bigg(\frac{1}{\sqrt{m}} \sigma'(\rmW \rvx_\alpha)_j \rva_j\rvx_{\alpha, k}\bigg) \cdot \frac{1}{\sqrt{m}} \sigma'(\rmW \rvx_\beta)_j \rva_j\rvx_{\beta, k}  \cdot \frac{1}{\sqrt{m}} \sigma(\rmW \rvx_\gamma)_v \\
		& = \frac{1}{m\sqrt{m}} \sigma'(\rmW \rvx_\alpha)_j \delta_{jv}\rvx_{\alpha, k} \cdot \sigma'(\rmW \rvx_\beta)_j \rva_j\rvx_{\beta, k} \cdot \sigma(\rmW \rvx_\gamma)_v \\
		& = \frac{1}{m\sqrt{m}} \sigma'(\rmW \rvx_\alpha)_j \sigma'(\rmW \rvx_\beta)_j \sigma(\rmW \rvx_\gamma)_j \rva_j\rvx_{\alpha, k}\rvx_{\beta, k} \\
		& = \frac{1}{m\sqrt{m}} \rvx_\alpha^\top\rvx_\beta \cdot \bigg\la a, \sigma'(\rmW \rvx_\alpha)\odot \sigma'(\rmW \rvx_\beta)\odot \sigma(\rmW \rvx_\gamma)\bigg\ra.
	\end{align*}
	By symmetry, the term $(\nabla f_\alpha)_{1jk} (\nabla^2 f_\beta)^{1jk}_{21v} (\nabla f_\gamma)_{21 v}$ is also equal to the above quantity. Finally, we calculate the $\ell = 2, s= 1$ term. Note that in this case, $j = 1$. Hence, we have
	\begin{align*}
		(\nabla^2 f_\alpha)^{21 k }_{1uv} (\nabla f_\beta)_{21 k} (\nabla f_\gamma)_{1uv}	& = \frac{\partial}{\partial \rmW_{uv}}\bigg( \frac{1}{\sqrt{m}} \sigma(\rmW \rvx_\alpha)_k\bigg)  \cdot \frac{1}{\sqrt{m}} \sigma(\rmW \rvx_\beta)_k  \cdot \frac{1}{\sqrt{m}} \sigma'(\rmW \rvx_\gamma)_u \rva_u\rvx_{\gamma, v}\\
		& = \frac{1}{m\sqrt{m}} \sigma'(\rmW \rvx_\alpha)_k \delta_{uk}\rvx_{\alpha, v}  \cdot \sigma(\rmW \rvx_\beta)_k \cdot \sigma'(\rmW \rvx_\gamma)_u \rva_u\rvx_{\gamma, v} \\
		& = \frac{1}{m\sqrt{m}} \sigma'(\rmW \rvx_\alpha)_k \sigma(\rmW \rvx_\beta)_k \sigma'(\rmW \rvx_\gamma)_k \rva_k\rvx_{\alpha, v}\rvx_{\gamma, v} \\
		& = \frac{1}{m\sqrt{m}}\rvx_\alpha^\top\rvx_\gamma \cdot \bigg\la a, \sigma'(\rmW \rvx_\alpha)\odot \sigma(\rmW \rvx_\beta)\odot \sigma'(\rmW \rvx_\gamma) \bigg\ra.
	\end{align*}
	A similar computation gives
	\[
		(\nabla f_\alpha)_{21 k } (\nabla^2 f_\beta)^{21 k}_{1uv} (\nabla f_\gamma)_{1uv}= \frac{1}{m\sqrt{m}}\rvx_\beta^\top\rvx_\gamma \cdot \bigg\la a, \sigma(\rmW \rvx_\alpha)\odot \sigma'(\rmW \rvx_\beta)\odot \sigma'(\rmW \rvx_\gamma) \bigg\ra.	
	\]
	Combining above terms proves Equation \eqref{eq: K_3_two_layer}.

	{\bf Computation of $\KKKK_4$.} Recall that
	\begin{align*}
		\KKKK_t = &   (\nabla^3 f_\alpha)^{\ell j k }_{suv, abc} (\nabla f_\beta)_{\ell j k} (\nabla f_\gamma)_{suv} (\nabla f_\xi)_{abc}\\
		& + (\nabla^2 f_\alpha)^{\ell j k }_{suv} (\nabla^2 f_\beta)^{\ell j k}_{abc} (\nabla f_\gamma)_{suv} (\nabla f_\xi)_{abc} \\
		& + (\nabla^2 f_\alpha)^{\ell j k }_{suv} (\nabla f_\beta)_{\ell j k} (\nabla^2 f_\gamma)^{suv}_{abc} (\nabla f_\xi)_{abc} \\
		& +  (\nabla^2 f_\alpha)^{\ell j k }_{abc} (\nabla^2 f_\beta)^{\ell j k}_{s u v} (\nabla f_\gamma)_{suv}  (\nabla f_\xi)_{abc} \\
		& +  (\nabla f_\alpha)_{\ell j k } (\nabla^3 f_\beta)^{\ell j k}_{s u v, abc} (\nabla f_\gamma)_{suv} (\nabla f_\xi)_{abc} \\
		& +  (\nabla f_\alpha)_{\ell j k } (\nabla^2 f_\beta)^{\ell j k}_{s u v} (\nabla^2 f_\gamma)^{suv}_{abc}  (\nabla f_\xi)_{abc}.
	\end{align*}
	We denote the six terms above as $\RN{1}, \RN{2}, \RN{3}, \RN{4}, \RN{5}, \RN{6}$.
	We first calculate some useful quantities. Recall that
	\[
		(\nabla f)_{21k} = \frac{1}{\sqrt{m}} \sigma(\rmW \rvx_k), \qquad (\nabla f)_{1jk} = \frac{1}{\sqrt{m}} \sigma'(\rmW  \rvx_j) \rva_j\rvx_k.
	\]
	In the calculation for $\KKK_t$, we have shown that
	\begin{align*}
		& (\nabla^2 f)^{2 1 k}_{2 1v}   = 0 \\
		& (\nabla^2 f)^{2 1 k}_{1 u v} = (\nabla^2 f)^{1 u v}_{2 1 k} = \frac{1}{\sqrt{m}} \sigma'(\rmW  \rvx_k) \delta_{uk}\rvx_v \\
		& (\nabla^2 f)^{1 j k}_{1 u v} =(\nabla^2 f)^{1 u v}_{1 j k } = \frac{1}{\sqrt{m}} \sigma''(\rmW  \rvx_j) \rva_j \delta_{ju}\rvx_k\rvx_v.
	\end{align*}
	For the third derivative, we note that the expression $(\nabla^3 f)^{\ell j k}_{suv, abc}$ is invariant to permutations of the three triplets $(\ell j k), (s u v), (a b c)$. Some algebra gives the following identities:
	\begin{align*}
		& (\nabla^3 f)^{2 1 k}_{2 1 v, 21 c} = (\nabla^3 f)^{2 1 k}_{2 1 v, 1 b c} = 0\\
		& (\nabla^3 f)^{1 j k}_{1 u v, 2 1 c} =  \frac{1}{\sqrt{m}} \sigma''(\rmW  \rvx_j) \delta_{ju} \delta_{jc}\rvx_k\rvx_v \\
		& (\nabla^3 f)^{1 j k}_{1 u v, 1 b c}= \frac{1}{\sqrt{m}} \sigma'''(\rmW  \rvx_j )\rva_j \delta_{bj}\delta_{ju}\rvx_k\rvx_v \rvx_c.
	\end{align*}
	We now calculate expression for $\KKKK_t$ based on different configurations of layer indices $(\ell, s, a) \in \{1, 2\}^{3}$. 
	\begin{enumerate}[itemsep=0mm]
		\item If $\ell = s = a = {2}$, or if $\ell = s = 2, a = 1$, then all six terms are zero. 
		\item If $\ell = a = 2, s = 1$, then $\RN{1}=\RN{2}=\RN{4}=\RN{5} = 0$. The third term is
		\begin{align*}
			\RN{3} & = (\nabla^2 f_\alpha)^{21k}_{1uv} (\nabla f_\beta)_{21k} (\nabla^2 f_\gamma)^{1uv}_{21c} (\nabla f_\xi)_{21c} \\
			& = \frac{1}{m^2} \sigma'(\rmW \rvx_\alpha)_{k} \delta_{uk}\rvx_{\alpha, v} \cdot \sigma(\rmW \rvx_\beta)_k \cdot \sigma'(\rmW \rvx_\gamma)_c \delta_{uc}\rvx_{\gamma, v} \cdot \sigma(\rmW \rvx_\xi)_c \\
			& = \frac{1}{m^2}\rvx_\alpha^\top\rvx_\gamma \bigg\la\mathbf{1}_m, \sigma'(\rmW \rvx_\alpha) \odot \sigma(\rmW \rvx_\beta) \odot \sigma'(\rmW \rvx_\gamma) \odot \sigma(\rmW \rvx_\xi) \bigg\ra.
		\end{align*}
		A similar calculation gives
		\[
			\RN{6} = \frac{1}{m^2}\rvx_\beta^\top\rvx_\gamma \bigg\la\mathbf{1}_m, \sigma(\rmW \rvx_\alpha) \odot \sigma'(\rmW \rvx_\beta) \odot \sigma'(\rmW \rvx_\gamma) \odot \sigma(\rmW \rvx_\xi) \bigg\ra.
		\]
		\item If $s = a = 2, \ell = 1$, then $\RN{1}=\RN{3}=\RN{5} = \RN{6} = 0$. And we have
		\begin{align*}
			\RN{2} & = (\nabla^2 f_\alpha )^{1jk}_{21 v} (\nabla^2 f_\beta)^{1jk}_{21c} (\nabla f_\gamma)_{21 v} (\nabla f_\xi)_{21c} \\
			& = \frac{1}{m^2}\rvx_\alpha^\top\rvx_\beta \bigg\la\mathbf{1}_m, \sigma'(\rmW \rvx_\alpha) \odot \sigma'(\rmW \rvx_\beta) \odot \sigma(\rmW \rvx_\gamma) \odot \sigma(\rmW \rvx_\xi) \bigg\ra.
		\end{align*}
		A similar calculation shows that $\RN{4} = \RN{2}$.
		\item If $\ell = s = 1, a = 2$, then we have
		\begin{align*}
			\RN{1} & = (\nabla^3 f_\alpha)^{1jk}_{1uv, 21 c} (\nabla f_\beta)_{1jk} (\nabla f_\gamma)_{1uv} (\nabla f_\xi)_{21 c} \\
			& = \frac{1}{m^2}\rvx_\alpha^\top\rvx_\gamma \cdot\rvx_\alpha^\top\rvx_\beta \bigg\la\rva^{\odot 2} , \sigma''(\rmW \rvx_\alpha) \odot \sigma'(\rmW \rvx_\beta) \odot \sigma'(\rmW \rvx_\gamma) \odot \sigma(\rmW \rvx_\xi) \bigg\ra.
		\end{align*}
		Meanwhile, we have
		\begin{align*}
			\RN{2} & = (\nabla^2 f_\alpha)^{1jk}_{1uv} (\nabla^2 f_\beta)^{1jk}_{21 c} (\nabla f_\gamma)_{1uv} (\nabla f_\xi)_{21c} \\
			& = \frac{1}{m^2}\rvx_\alpha^\top\rvx_\beta \cdot\rvx_\alpha^\top\rvx_\gamma \bigg\la\rva^{\odot 2} , \sigma''(\rmW \rvx_\alpha) \odot \sigma'(\rmW \rvx_\beta) \odot \sigma'(\rmW \rvx_\gamma) \odot \sigma(\rmW \rvx_\xi) \bigg\ra \\
			& = \RN{1}.
		\end{align*}
		A similar calculation shows that $\RN{3} = \RN{2} = \RN{1}$. On the other hand, it's easy to check that
		\[
			\RN{4} = \RN{5} = \RN{6} = 
			\frac{1}{m^2}\rvx_\beta^\top\rvx_\alpha \cdot\rvx_\beta^\top\rvx_\gamma \bigg\la\rva^{\odot 2} , \sigma'(\rmW \rvx_\alpha) \odot \sigma''(\rmW \rvx_\beta) \odot \sigma'(\rmW \rvx_\gamma) \odot \sigma(\rmW \rvx_\xi) \bigg\ra.
		\]
		\item If $\ell = a = 1, s = 2$, then one can check that
		\[
			\RN{1} = \RN{4} = 
			\frac{1}{m^2}\rvx_\alpha^\top\rvx_\beta \cdot\rvx_\alpha^\top\rvx_\xi \bigg\la\rva^{\odot 2} , \sigma''(\rmW \rvx_\alpha) \odot \sigma'(\rmW \rvx_\beta) \odot \sigma(\rmW \rvx_\gamma) \odot \sigma'(\rmW \rvx_\xi) \bigg\ra,
		\]
		and that
		\[
			\RN{2} = \RN{5} = 
			\frac{1}{m^2}\rvx_\beta^\top\rvx_\alpha \cdot\rvx_\beta^\top\rvx_\xi \bigg\la\rva^{\odot 2} , \sigma'(\rmW \rvx_\alpha) \odot \sigma''(\rmW \rvx_\beta) \odot \sigma(\rmW \rvx_\gamma) \odot \sigma'(\rmW \rvx_\xi) \bigg\ra,
		\]
		On the other hand, we have
		\[
			\RN{3} = \RN{6}
			\frac{1}{m^2}\rvx_\alpha^\top\rvx_\beta \cdot\rvx_\gamma^\top\rvx_\xi \bigg\la\rva^{\odot 2} , \sigma'(\rmW \rvx_\alpha) \odot \sigma'(\rmW \rvx_\beta) \odot \sigma'(\rmW \rvx_\gamma) \odot \sigma'(\rmW \rvx_\xi) \bigg\ra.
		\]
		\item If $s = a = 1, \ell = 2$, then we have
		\begingroup
		\allowdisplaybreaks
		\begin{align*}
			\RN{1} & = 
			\frac{1}{m^2}\rvx_\alpha^\top\rvx_\gamma \cdot\rvx_\alpha^\top\rvx_\xi \bigg\la\rva^{\odot 2} , \sigma''(\rmW \rvx_\alpha) \odot \sigma(\rmW \rvx_\beta) \odot \sigma'(\rmW \rvx_\gamma) \odot \sigma'(\rmW \rvx_\xi) \bigg\ra \\
			\RN{2} & = 
			\frac{1}{m^2}\rvx_\alpha^\top\rvx_\gamma \cdot\rvx_\beta^\top\rvx_\xi \bigg\la\rva^{\odot 2} , \sigma'(\rmW \rvx_\alpha) \odot \sigma'(\rmW \rvx_\beta) \odot \sigma'(\rmW \rvx_\gamma) \odot \sigma'(\rmW \rvx_\xi) \bigg\ra \\ 
			\RN{3} & = 
			\frac{1}{m^2}\rvx_\gamma^\top\rvx_\alpha \cdot\rvx_\gamma^\top\rvx_\xi \bigg\la\rva^{\odot 2} , \sigma'(\rmW \rvx_\alpha) \odot \sigma(\rmW \rvx_\beta) \odot \sigma''(\rmW \rvx_\gamma) \odot \sigma'(\rmW \rvx_\xi) \bigg\ra \\ 
			\RN{4} & = 
			\frac{1}{m^2}\rvx_\alpha^\top\rvx_\xi \cdot\rvx_\beta^\top\rvx_\gamma \bigg\la\rva^{\odot 2} , \sigma'(\rmW \rvx_\alpha) \odot \sigma'(\rmW \rvx_\beta) \odot \sigma'(\rmW \rvx_\gamma) \odot \sigma'(\rmW \rvx_\xi) \bigg\ra \\ 
			\RN{5} & = 
			\frac{1}{m^2}\rvx_\beta^\top\rvx_\gamma \cdot\rvx_\beta^\top\rvx_\xi \bigg\la\rva^{\odot 2} , \sigma(\rmW \rvx_\alpha) \odot \sigma''(\rmW \rvx_\beta) \odot \sigma'(\rmW \rvx_\gamma) \odot \sigma'(\rmW \rvx_\xi) \bigg\ra \\ 
			\RN{6} & = 
			\frac{1}{m^2}\rvx_\gamma^\top\rvx_\beta \cdot\rvx_\gamma^\top\rvx_\xi \bigg\la\rva^{\odot 2} , \sigma(\rmW \rvx_\alpha) \odot \sigma'(\rmW \rvx_\beta) \odot \sigma''(\rmW \rvx_\gamma) \odot \sigma'(\rmW \rvx_\xi) \bigg\ra.
		\end{align*}
		\endgroup
		\item If $\ell = a = s = 1$, then we have
		\begin{align*}
			\RN{1} & = (\nabla^3 f_\alpha)^{1jk}_{1uv, 1bc} (\nabla f_\beta)_{1jk} (\nabla f_\gamma)_{1uv} (\nabla f_\xi)_{1bc} \\
			& = \frac{1}{m^2} \sigma'''(\rmW \rvx_\alpha)_j \rva_j \delta_{bj} \delta_{ju}\rvx_{\alpha, c}\rvx_{\alpha, v}\rvx_{\alpha, k} \cdot \sigma'(\rmW \rvx_\beta)_j \rva_j\rvx_{\beta, k} \cdot \sigma'(\rmW \rvx_\gamma)_u \rva_u\rvx_{\gamma, v}  \cdot \sigma'(\rmW \rvx_\xi)_b \rva_b\rvx_{\xi, c} \\
			& = 
			\frac{1}{m^2}\rvx_\alpha^\top\rvx_\xi \cdot\rvx_\alpha^\top\rvx_\gamma \cdot\rvx_\alpha^\top\rvx_\beta \bigg\la\rva^{\odot 4} , \sigma'''(\rmW \rvx_\alpha) \odot \sigma'(\rmW \rvx_\beta) \odot \sigma'(\rmW \rvx_\gamma) \odot \sigma'(\rmW \rvx_\xi) \bigg\ra.
		\end{align*}
		Meanwhile, we have
		\begin{align*}
			\RN{2} & = (\nabla^2 f_\alpha)^{1jk}_{1uv}  (\nabla^2 f_{\beta})^{1jk}_{1bc} (\nabla f_\gamma)_{1uv} (\nabla f_\xi)_{1bc} \\
			& = \frac{1}{m^2} \sigma''(\rmW \rvx_\alpha)_j \rva_j \delta_{ju}\rvx_{\alpha, k}\rvx_{\alpha, v} \cdot \sigma''(\rmW \rvx_\beta)_j \rva_j \delta_{jb}\rvx_{\beta, k}\rvx_{\beta, c} \cdot \sigma'(\rmW \rvx_\gamma)_u \rva_u\rvx_{\gamma, v} \cdot \sigma'(\rmW \rvx_\xi)_b \rva_b\rvx_{\xi, c} \\
			& = 	
			\frac{1}{m^2}\rvx_\alpha^\top\rvx_\beta \cdot\rvx_\alpha^\top\rvx_\gamma \cdot\rvx_\beta^\top\rvx_\xi \bigg\la\rva^{\odot 4} , \sigma''(\rmW \rvx_\alpha) \odot \sigma''(\rmW \rvx_\beta) \odot \sigma'(\rmW \rvx_\gamma) \odot \sigma'(\rmW \rvx_\xi) \bigg\ra.
		\end{align*}
		The other terms are calculated similarly:
		\begin{align*}
			\RN{3} & = 
			\frac{1}{m^2}\rvx_\alpha^\top\rvx_\gamma \cdot\rvx_\alpha^\top\rvx_\beta \cdot\rvx_\gamma^\top\rvx_\xi \bigg\la\rva^{\odot 4} , \sigma''(\rmW \rvx_\alpha) \odot \sigma'(\rmW \rvx_\beta) \odot \sigma''(\rmW \rvx_\gamma) \odot \sigma'(\rmW \rvx_\xi) \bigg\ra\\
			\RN{4} & = 
			\frac{1}{m^2}\rvx_\alpha^\top\rvx_\beta \cdot\rvx_\alpha^\top\rvx_\xi \cdot\rvx_\beta^\top\rvx_\gamma \bigg\la\rva^{\odot 4} , \sigma''(\rmW \rvx_\alpha) \odot \sigma''(\rmW \rvx_\beta) \odot \sigma'(\rmW \rvx_\gamma) \odot \sigma'(\rmW \rvx_\xi) \bigg\ra\\ 
			\RN{5} & =
			\frac{1}{m^2}\rvx_\beta^\top\rvx_\alpha \cdot\rvx_\beta^\top\rvx_\gamma \cdot\rvx_\beta^\top\rvx_\xi \bigg\la\rva^{\odot 4} , \sigma'(\rmW \rvx_\alpha) \odot \sigma'''(\rmW \rvx_\beta) \odot \sigma'(\rmW \rvx_\gamma) \odot \sigma'(\rmW \rvx_\xi) \bigg\ra\\ 
			\RN{6} & = 
			\frac{1}{m^2}\rvx_\beta^\top\rvx_\gamma \cdot\rvx_\beta^\top\rvx_\alpha \cdot\rvx_\gamma^\top\rvx_\xi \bigg\la\rva^{\odot 4} , \sigma'(\rmW \rvx_\alpha) \odot \sigma''(\rmW \rvx_\beta) \odot \sigma''(\rmW \rvx_\gamma) \odot \sigma'(\rmW \rvx_\xi) \bigg\ra.
		\end{align*}
	\end{enumerate}
	Putting the above terms together gives Equation \eqref{eq: K_4_two_layer}.
\end{proof}

\subsection{Expected Values w.r.t. Gaussian Initialization}

We now consider the expected values of $K^{(r)}_t$ at initialization, where both $\rva_0$ and $\rmW_0$ have i.i.d. $\calN(0, 1)$ entries. We will focus the ReLU activation:
\begin{equation}
	\sigma(x ) = \max(0, x).
\end{equation}
Technically, $\sigma(\cdot)$ only has a subdifferential at zero, but since Gaussian initialization puts zero mass at this point, we can safely write $\sigma'(x ) = \mathbbm{1}\{x\geq 0\}$. Moreover, we have $\sigma''(x ) = \delta(x )$, where $\delta(\cdot)$ is the Dirac delta function, and $\sigma'''(x )= \delta'(x )$, where $\delta'(\cdot)$ is the distributional derivative of $\delta(\cdot)$. In this sense, many terms in $\KKKK_t$ are not well-defined, if we don't take expectation. For example, the following terms are not well-defined functions:
\begin{align*}
	& \bigg\la\rva_t^{\odot 4}, \sigma'''(\rmW\rvx_\alpha) \odot \sigma'(\rmW \rvx_\beta) \odot \sigma'(\rmW \rvx_\gamma) \odot \sigma'(\rmW \rvx_\xi) \bigg\ra \\
	& \bigg\la\rva^{\odot 4}, \sigma''(\rmW \rvx_\alpha) \odot \sigma''(\rmW \rvx_\beta) \odot \sigma'(\rmW \rvx_\gamma) \odot \sigma'(\rmW \rvx_\xi)\bigg\ra \\
	&  \bigg\la\rva^{\odot 4}, \sigma''(\rmW \rvx_\alpha) \odot \sigma(\rmW \rvx_\beta) \odot \sigma'(\rmW \rvx_\gamma) \odot \sigma'(\rmW \rvx_\xi)\bigg\ra.
\end{align*}
So it is necessary to integrate over the Gaussian measure to actually make sense of the above expressions.

On the other hand, the following expressions are well-defined functions:

\begin{align*}
		& \bigg\la\rva^{\odot 2},\sigma'(\rmW \rvx_\alpha) \odot \sigma'(\rmW \rvx_\beta) \odot \sigma'(\rmW \rvx_\gamma) \odot \sigma'(\rmW \rvx_\xi)   \bigg\ra \\
		& \bigg\la \mathbf{1}_m, \sigma'(\rmW \rvx_\alpha) \odot \sigma(\rmW \rvx_\beta) \odot \sigma'(\rmW \rvx_\gamma) \odot \sigma(\rmW \rvx_\xi)  \bigg\ra, 
\end{align*}
because there is no expressions like $\delta(\cdot)$ and $\delta'(\cdot)$.

We now calculate the expectation of $\KK_0, \KKK_0$ and $\KKKK_0$ under Gaussian initialization. First, let us note that $\bbE \KKK_0 = 0$. In fact, since the $r$-th moment of $\calN(0, 1)$ is zero for odd $r$, we have $\bbE K^{(r)} = 0$ for any odd $r$.

\subsubsection{Expectation of the second-order kernel}
\label{subsubsection:label-agnostic-kernel}
For $\KK_0$, we need to calculate the following two quantities:
\[
	\bbE \bigg\la \sigma'(\rmW \rvx_\alpha)\odot\rva_t, \sigma'(\rmW \rvx_\beta)\odot\rva_t  \bigg\ra,  \qquad \bbE \bigg\la \sigma(\rmW \rvx_\alpha), \sigma(\rmW \rvx_\beta) \bigg\ra.
\]
For the first term, we have 
\[
	\bbE \bigg\la \sigma'(\rmW \rvx_\alpha)\odot\rva_t, \sigma'(\rmW \rvx_\beta)\odot\rva_t  \bigg\ra = m \bbE_{Z\sim\calN(0, I_d)} \bigg[\sigma'(\rvx_\alpha^\top Z)\sigma'(\rvx_\beta^\top Z) \bigg],
\]
whereas for the second term, we have
\[
	\bbE \bigg\la \sigma(\rmW \rvx_\alpha), \sigma(\rmW \rvx_\beta) \bigg\ra = m \bbE_{Z\sim\calN(0, I_d)} \bigg[\sigma(\rvx_\alpha^\top Z)\sigma(\rvx_\beta^\top Z)\bigg].
\]
The above quantities are calculated in many literature (see, e.g., \citealt{cho2010large,arora2019exact,bietti2019inductive}). Let $\Delta_{\alpha\beta}\in[0, \pi)$ be the angle between $x_\alpha$ and $x_\beta$. We have
\begin{align*}
	\bbE_{Z\sim\bbN(0, I_d)} \bigg[\sigma'(\rvx_\alpha^\top Z)\sigma'(\rvx_\beta^\top Z) \bigg] & = \frac{\pi - \Delta_{\alpha\beta}}{2\pi}\\
	\bbE_{Z\sim\calN(0, I_d)} \bigg[\sigma(\rvx_\alpha^\top Z)\sigma(\rvx_\beta^\top Z)\bigg] & = \frac{\|x_\alpha\|_2\|x_\beta\|_2 }{2\pi} \cdot \bigg(\sin\Delta_{\alpha\beta} + (\pi-\Delta_{\alpha\beta})\cos\Delta_{\alpha\beta}\bigg).
\end{align*}	
Hence, we have
\begin{equation*}
	\label{eq: K_2_two_layer_expectation}
	\bbE \KK_0(\rvx_\alpha,\rvx_\beta) =\rvx_\alpha^\top\rvx_\beta \cdot \frac{\pi - \Delta_{\alpha\beta}}{2\pi} +  \|x_\alpha \|_2 \|\rvx_\beta\|_2 \cdot \frac{\sin\Delta_{\alpha\beta} + (\pi - \Delta_{\alpha\beta})\cos\Delta_{\alpha\beta}}{2\pi}.
\end{equation*}

\subsubsection{Expectation of the fourth-order kernel}
We now try to compute $\bbE \KKKK_0$. Inspecting Equation \eqref{eq: K_4_two_layer}, we notice that it suffices to calculate the expectation of the following quantities: 
\begin{align*}
	\RN{1}  &= \bigg\la\rva^{\odot 2},\sigma'(\rmW \rvx_\alpha) \odot \sigma'(\rmW \rvx_\beta) \odot \sigma'(\rmW \rvx_\gamma) \odot \sigma'(\rmW \rvx_\xi)   \bigg\ra \\
	\RN{2}  &= \bigg\la \mathbf{1}_m, \sigma'(\rmW \rvx_\alpha) \odot \sigma'(\rmW \rvx_\beta) \odot \sigma(\rmW \rvx_\gamma) \odot \sigma(\rmW \rvx_\xi)  \bigg\ra \\
	\RN{3}  &=  \bigg\la\rva^{\odot 2}, \sigma''(\rmW \rvx_\alpha) \odot \sigma(\rmW \rvx_\beta) \odot \sigma'(\rmW \rvx_\gamma) \odot \sigma'(\rmW \rvx_\xi)\bigg\ra \\
	\RN{4}  &= \bigg\la\rva^{\odot 4}, \sigma''(\rmW \rvx_\alpha) \odot \sigma''(\rmW \rvx_\beta) \odot \sigma'(\rmW \rvx_\gamma) \odot \sigma'(\rmW \rvx_\xi)\bigg\ra \\
	\RN{5}  &= \bigg\la\rva^{\odot 4}, \sigma'''(\rmW\rvx_\alpha) \odot \sigma'(\rmW \rvx_\beta) \odot \sigma'(\rmW \rvx_\gamma) \odot \sigma'(\rmW \rvx_\xi) \bigg\ra.
\end{align*}
The rest of the terms can be calculated similarly.
~\\

{\bf The first term.} For the first term, we have
\[
	\bbE \bigg\la\rva^{\odot 2},\sigma'(\rmW \rvx_\alpha) \odot \sigma'(\rmW \rvx_\beta) \odot \sigma'(\rmW \rvx_\gamma) \odot \sigma'(\rmW \rvx_\xi)   \bigg\ra = m \bbP \bigg(\rvx_\alpha^\top Z \geq 0,\rvx_\beta^\top Z \geq 0,\rvx_\gamma^\top Z \geq 0,\rvx_\xi^\top Z \geq 0\bigg),
\]
Note that this term only depends on the angles, so we can WLOG assume that all four vectors lies on the sphere. We reduce this $d$-dimensional integral to a four-dimensional one. For any orthogonal matrix $Q$, we have
\[
	\bbP \bigg(\rvx_\alpha^\top Z \geq 0,\rvx_\beta^\top Z \geq 0,\rvx_\gamma^\top Z \geq 0,\rvx_\xi^\top Z \geq 0\bigg) = \bbP \bigg((Q\rvx_\alpha)^\top Z \geq 0, (Q\rvx_\beta)^\top Z \geq 0, (Q\rvx_\gamma)^\top Z \geq 0, (Q\rvx_\xi)^\top Z \geq 0\bigg).
\]
We choose a $Q$ s.t $(Q\rvx_{\alpha})_i = 0$ for $i \geq 2$, $(Q\rvx_\beta)_i = 0$ for $i \geq 3$, $(Q\rvx_\gamma)_i  = 0$ for $i \geq 4$, and $(Q\rvx_\xi)_i = 0$ for $i \geq 5$. Moreover, we require $(Q\rvx_\alpha)_1 =1$. Those requirements specify a unique (up to flips in one direction) rotation matrix $Q$. In order the preserve the angles, we necessarily have
\begin{align*}
	& Q\rvx_\alpha  = 
	\begin{pmatrix}
	1 & 0 & 0 & 0 & 0 & \cdots & 0
	\end{pmatrix}^\top \\
	& (Q\rvx_\beta)^\top (Q\rvx_\alpha)   = \cos\Delta_{\alpha\beta}, \ \ \  \|Q\rvx_\beta \|_2 = 1 \\
	& (Q\rvx_\gamma)^\top (Q\rvx_\alpha)  =  \cos\Delta_{\alpha\gamma} , \ \ \
	(Q\rvx_\gamma)^\top (Q\rvx_\beta)  = \cos \Delta_{\beta \gamma}, \ \ \  \|Q\rvx_\gamma \|_2 = 1\\
	& (Q\rvx_\xi)^\top (Q\rvx_\alpha)  = \cos \Delta_{\alpha\xi} , \ \ \ (Q\rvx_\xi)^\top (Q\rvx_\beta) = \cos \Delta_{\beta\xi} , \ \ \ (Q\rvx_\xi)^\top (Q\rvx_\gamma) = \cos \Delta_{\gamma\xi}, \ \ \ \| Q\rvx_\xi\|_2 = 1.
\end{align*}
Solving the above system of equations gives
\begin{align*}
	(Q\rvx_\beta) & = 
	[\cos\Delta_{\alpha \beta} , \ \sin\Delta_{\alpha\beta} , \  0 , \  \cdots, \  0]^\top \\
	(Q\rvx_\gamma) & = \bigg[
	\cos\Delta_{\alpha\gamma}   ,  \
	(\sin^{-1} \Delta_{\alpha\beta})\cdot ({\cos\Delta_{\beta\gamma} - \cos\Delta_{\alpha\beta}\cos\Delta_{\alpha\gamma}}), \ \sqrt{\sin^2\Delta_{\alpha\gamma} - (Q\rvx_\gamma)_2^2} , \ 0, \ \cdots, \ 0 )^\top \\
	(Q\rvx_\xi) & = \bigg[
	\cos\Delta_{\alpha\xi}, \
	(\sin^{-1} \Delta_{\alpha\beta})\cdot ({\cos\Delta_{\beta\xi} - \cos\Delta_{\alpha\beta}\cos\Delta_{\alpha\xi}}) , \ \frac{\cos\Delta_{\gamma\xi} - (Q\rvx_\gamma)_2 (Q\rvx_\xi)_2}{(Q\rvx_\gamma)_3} , \\
	& \qquad \qquad \qquad \qquad \qquad \qquad \qquad\qquad\qquad\qquad  \sqrt{\sin^2\Delta_{\alpha\xi} - (Q\rvx_\xi)_2^2 - (Q\rvx_\xi)_3^2}, \ 0, \ \cdots, \ 0
	\bigg]^\top.
\end{align*}
Let $v_\alpha, v_\beta, v_\gamma, v_\xi$ be the vectors composed of the first four coordinates of $Q\rvx_\alpha, Q\rvx_\beta, Q\rvx_\gamma, Q\rvx_\xi$, respectively. Then we have
\[
	\bbP \bigg(\rvx_\alpha^\top Z \geq 0,\rvx_\beta^\top Z \geq 0,\rvx_\gamma^\top Z \geq 0,\rvx_\xi^\top Z \geq 0\bigg) = \bbP \bigg(v_\alpha^\top Z \geq 0, v_\beta^\top Z \geq 0, v_\gamma^\top Z \geq 0, v_\xi^\top Z \geq 0\bigg),
\]
where with a slight abuse of notation, the vector $Z\sim\calN(0, I_4)$ is now a four-dimensional standard Gaussian. Let
\begin{equation}
	\label{eq: supporting_vectors_in_four_dim}
	V =  V_{\alpha, \beta, \gamma, \xi} = 
	\begin{pmatrix}
		v_\alpha^\top \\
		v_\beta^\top \\
		v_\gamma^\top \\
		v_\xi
	\end{pmatrix} \in \bbR^{4\times 4}.
\end{equation}	
Then the quantity of interest becomes $\bbP(VZ \in \calH)$, where $\calH$ is the positive orthant of $\bbR^4$. It's clear from the construction that $V$ is invertible, so we are interested in $\bbP(Z\in V^{-1}\calH)$. The set $V^{-1}\calH$ is the \emph{positive hull} made by the four columns of $V^{-1}$. Since the law of $Z$ is spherically symmetric, the measure of $V^{-1}H$ under the law of $Z$ is the fraction of the unit sphere $\bbS^3$  in this hull, which is equal to $\Omega/(2\pi^2)$, where $2\pi^2$ is the surface area of $\bbS^3$, and $\Omega$ is the \emph{solid angle} for the four column vectors of $V^{-1}$. The solid angle for $r$ vectors has analytical formulas if $r \leq 3$. For $r = 4$ and higher dimensional, there is a formula in terms of multivariate Taylor series (see, e.g., \citealt{aomoto1977analytic,ribando2006measuring}), but no closed-form formulas are known to the best of our knowledge. However, the probability we are interested in can be efficiently simulated by law of large numbers, because we have reduce the $d$-dimensional Gaussian integral to a four-dimensional one. 

In summary, we have the following expression:
\[
	\bbE \RN{1} = m \bbP(V Z \succeq 0), \qquad Z\sim\calN(0, I_4).
\]
~\\

{\bf The second term.} For the second term, we have
\[
	\bbE \bigg\la \mathbf{1}_m, \sigma'(\rmW \rvx_\alpha) \odot \sigma'(\rmW \rvx_\beta) \odot \sigma(\rmW \rvx_\gamma) \odot \sigma(\rmW \rvx_\xi)  \bigg\ra  = m \bbE \bigg[x_\gamma^\top Z\cdot\rvx_\xi^\top Z\cdot \mathbbm{1}\bigg\{x_\alpha^\top Z \geq 0,\rvx_\beta^\top Z \geq 0,\rvx_\gamma^\top Z \geq 0,\rvx_\xi^\top Z \geq 0\bigg\}\bigg].
\]
Similarly, we have
\begin{align*}
	& \bbE_{Z\sim\calN(0, I_d)} \bigg[x_\gamma^\top Z\cdot\rvx_\xi^\top Z\cdot \mathbbm{1}\bigg\{x_\alpha^\top Z \geq 0,\rvx_\beta^\top Z \geq 0,\rvx_\gamma^\top Z \geq 0,\rvx_\xi^\top Z \geq 0\bigg\}\bigg] \\
	& =  \|\rvx_\gamma\|_2 \|\rvx_\xi \|_2\cdot \bbE_{Z\sim\calN(0, I_4)} \bigg[v_\gamma^\top Z \cdot v_\xi^\top Z \cdot \mathbbm{1}\{V_{\alpha, \beta, \gamma, \xi}Z \in \calH\}\bigg],
\end{align*}	
which gives the following expression:
\[
	\bbE \RN{2} = m \|\rvx_\gamma\|_2 \|\rvx_\xi \|_2\cdot \bbE \bigg[v_\gamma^\top Z \cdot v_\xi^\top Z \cdot \mathbbm{1}\{V_{\alpha, \beta, \gamma, \xi}Z \succeq 0\}\bigg] \qquad Z\sim\calN(0, I_4).
\]
~\\

{\bf The third term.} For this term, we have
\[
	\bbE \RN{3} = 3m \bbE \bigg[\delta(\rvx_\alpha^\top Z)\cdot\rvx_\beta^\top Z \cdot \mathbbm{1}\bigg\{x_\beta^\top Z \geq 0,\rvx_\gamma^\top Z \geq 0,\rvx_\xi^\top Z \geq 0\bigg\}\bigg].
\]
Using the same rotation trick, we have
\[
	\RN{3} = 3m \bbE\bigg[ \delta(\|x_\alpha\|_2 v_\alpha^\top Z ) \cdot \|x_\beta \|_2 v_\beta^\top Z \cdot \mathbbm{1}\bigg\{ v_\beta^\top Z \geq 0, v_\gamma^\top Z \geq 0, v_\xi^\top Z \geq 0  \bigg\} \bigg] ,\qquad Z\sim\calN(0, I_4).
\]
Since $v_\alpha = (1, 0, 0, 0)^\top$, we have
\begin{align*}
	& \bbE \bigg[\delta(\|x_\alpha \|_2 Z_1) \cdot v_\beta^\top Z \cdot \mathbbm{1}\bigg\{ v_\beta^\top Z\geq 0, v_\gamma^\top Z \geq 0, v_\xi^\top Z \geq 0\bigg\}\bigg] \\
	& = \bbE_{Z_2, Z_3, Z_4} \bigg[ \int_{-\infty}^\infty dz_1 \frac{1}{\sqrt{2\pi}} e^{-z_1^2/2} \delta(\|x_\alpha \|_2 z_1) h(z_1, Z_2, Z_3, Z_4) \bigg] \\
	& = \bbE_{Z_2, Z_3, Z_4} \bigg[ \frac{1}{\|\rvx_\alpha\|_2\sqrt{2\pi}} h(0, Z_2, Z_3, Z_4) \bigg] \\
	& = \frac{1}{\|x_\alpha\|_2\sqrt{2\pi}} \bbE\bigg[ v_{\beta, -1}^\top Z_{-1}\cdot \mathbbm{1}\bigg\{ v_{\beta, -1}^\top Z_{-1} \geq 0, v_{\gamma, -1}^\top Z_{-1}\geq 0, v_{\xi, -1}^\top Z_{-1}\geq 0 \bigg\}\bigg],
\end{align*}
where we let $h(Z_1, Z_2, Z_3, Z_4) =  v_\beta^\top Z \cdot \mathbbm{1}\bigg\{ v_\beta^\top Z\geq 0, v_\gamma^\top Z \geq 0, v_\xi^\top Z \geq 0\bigg\}$, and in the second equality, we used the fact that, for $c\neq 0$,
\[
	\int_{-\infty}^{\infty} f(\rvx )\delta(cx) dx = \int_{-\infty}^\infty \frac{1}{c} f(\frac{cx}{c}) \delta(cx) d cx = \frac{1}{c} \int_{\sgn(u)\cdot (-\infty)}^{\sgn(u)\cdot (+ \infty)} f(\frac{u}{c})\delta(u) du  = \frac{1}{|c|} \int_{-\infty}^\infty f(\frac{u}{c}) \delta(u) du = \frac{f(0)}{|c|}.
\]
Note that the above computation isn't too messy, because we rotate $Z$ to align with $x_\alpha$, and the delta function appears only at the $\alpha$ location. Other terms in $\KKKK_t$ with similar structures as $\RN{3}$ should be handled similarly (i.e., rotate $Z$ to align with the axis where the delta function appears).

In summary, we have
\[
	\bbE \RN{3} = \frac{3m}{  \sqrt{2\pi}} \cdot \frac{\|x_\beta\|_2}{\|x_\alpha \|_2} \cdot \bbE\bigg[ v_{\beta, -1}^\top Z_{-1} \cdot \mathbbm{1}\bigg\{ v_{\beta, -1}^\top Z_{-1} \geq 0, v_{\gamma, -1}^\top Z_{-1}\geq 0, v_{\xi, -1}^\top Z_{-1}\geq 0 \bigg\}\bigg].
\]

~\\

{\bf The fourth term.} We have
\allowdisplaybreaks
\begingroup
\begin{align*}
	\bbE \RN{4} & = 3m  \bbE\bigg[\delta(\rvx_\alpha^\top Z) \delta(\rvx_\beta^\top Z)\mathbbm{1}\bigg\{\rvx_\gamma^\top Z \geq 0,\rvx_\xi^\top Z \geq 0 \bigg\}\bigg] \\
	& = 3m \bbE \bigg[\delta(\|x_\alpha \|_2\cdot Z_1)\cdot \delta\bigl(\|x_\beta \|_2 \cdot (v_{\beta, 1} Z_1 + v_{\beta, 2} Z_2)\bigr) \cdot \mathbbm{1}\bigg\{\rvx_\gamma^\top Z \geq 0,\rvx_\xi^\top Z \geq 0 \bigg\}\bigg] \\
	& = 3m \bbE_{Z_3, Z_4} \bigg[\int_{-\infty}^\infty dz_2 \frac{1}{\sqrt{2\pi}}e^{-z_2^2/2} \int_{-\infty}^\infty dz_1 \frac{1}{\sqrt{2\pi}}e^{-z_1^2/2} \cdot \delta(\|\rvx_\alpha\|_2 z_1)\cdot \delta\bigl(\|x_\beta \|_2 \cdot (v_{\beta, 1} z_1 + v_{\beta, 2} z_2)\bigr) \\
	& \qquad \cdot \mathbbm{1}\bigg\{ v_{\gamma, 1} z_1 + v_{\gamma, 2} z_2 + v_{\gamma, 3} Z_3 + v_{\gamma, 4} Z_4\geq 0, v_{\xi, 1} z_1 + v_{\xi, 2} z_2 + v_{\xi, 3} Z_3 + v_{\xi, 4} Z_4\geq 0 \bigg\} \bigg]  \\
	& = \frac{3m}{\|x_\alpha\|_2\cdot 2\pi} \cdot \bbE_{Z_3, Z_4}  \bigg[\int_{-\infty}^\infty dz_2 e^{-z_2^2/2} \cdot \delta(\|x_\beta\|_2 v_{\beta, 2} z_2) \\
	& \qquad \cdot  \mathbbm{1}\bigg\{  v_{\gamma, 2} z_2 + v_{\gamma, 3} Z_3 + v_{\gamma, 4} Z_4\geq 0,  v_{\xi, 2} z_2 + v_{\xi, 3} Z_3 + v_{\xi, 4} Z_4\geq 0 \bigg\} \bigg] \\
	& = \frac{3m}{2\pi \|x_\alpha \|_2\|x_\beta\|_2|v_{\beta, 2}|}  \bbP(v_{\gamma, 3}Z_3 + v_{\gamma, 4}Z_4 \geq 0, v_{\xi, 3}Z_3 + v_{\xi, 4} Z_4 \geq 0).
\end{align*}	
\endgroup
The probability in the RHS can be calculated explicitly. By spherical symmetry of Gaussian, we have
\[
	\bbP(v_{\gamma, 3}Z_3 + v_{\gamma, 4}Z_4 \geq 0, v_{\xi, 3}Z_3 + v_{\xi, 4} Z_4 \geq 0) = \frac{\pi - \angle(v_{\gamma, 3:4}, v_{\xi, 3:4})}{2\pi},
\]
where $\angle(v_{\gamma, 3:4}, v_{\xi, 3:4}) = \arccos(\frac{v_{\gamma, 3:4}^\top v_{\xi, 3:4}}{\|v_{\gamma, 3:4} \|_2\|v_{\xi, 3:4}\|_2})\in [0, \pi)$ is the angle between the two vectors $(v_{\gamma , 3}, v_{\gamma, 4}), (v_{\xi, 3}, v_{\xi ,4})$. Hence, we arrive at
\[
	\bbE \RN{4} =  \frac{3m [\pi - \angle(v_{\gamma, 3:4}, v_{\xi, 3:4})]}{4\pi^2 \|x_\alpha \|_2\|x_\beta\|_2|v_{\beta, 2}|}.
\]
~\\

{\bf The fifth term.} We have 
\begin{align*}
	\bbE \RN{5} & = 3m \bbE \bigg[\delta'(\rvx_\alpha^\top Z)\cdot \mathbbm{1}\bigg\{x_\beta^\top Z \geq 0,\rvx_\gamma^\top Z \geq 0,\rvx_\xi^\top Z \geq 0 \bigg\}\bigg] \\
	& = 3m \bbE \bigg[\delta'(\|x_\alpha\|_2 Z_1) \cdot \mathbbm{1}\bigg\{v_\beta^\top Z \geq 0, v_\gamma^\top Z \geq 0, v_\xi^\top Z \geq 0\bigg\}\bigg] \\
	& =  \frac{3m}{\sqrt{2\pi}} \bbE_{Z_2, Z_3, Z_4} \bigg[\int_{-\infty}^{\infty} dz_1   \cdot \delta'(\|\rvx_\alpha\|_2 z_1) \cdot e^{-z_1^2/2} \\
	& \qquad \cdot \sigma'(v_{\beta, 1}z_1 + v_{\beta, -1}^\top Z_{-1})\sigma'(v_{\gamma, 1}z_1 + v_{\gamma, -1}^\top Z_{-1})\sigma'(v_{\xi, 1}z_1 + v_{\xi, -1}^\top Z_{-1})\bigg] \\
	& = \frac{3m}{\|x_\alpha \|_2\sqrt{2\pi}} \bbE_{Z_1, Z_2, Z_3} \bigg[ \int_{-\infty}^\infty du \cdot \delta'(u) \cdot e^{-{u^2}/{2\|x_\alpha\|_2^2}}\\
	& \qquad \cdot \sigma'\bigg(\frac{v_{\beta, 1}u}{\|x_\alpha\|_2} + v_{\beta, -1}^\top Z_{-1}\bigg)\sigma'\bigg(\frac{v_{\gamma, 1}u}{\|x_\alpha\|_2} + v_{\gamma, -1}^\top Z_{-1}\bigg)\sigma'\bigg(\frac{v_{\xi, 1}u}{\|x_\alpha\|_2} + v_{\xi, -1}^\top Z_{-1}\bigg)\bigg] \\
	& = -\frac{3m}{\|x_\alpha \|_2\sqrt{2\pi}}\bbE_{Z_2, Z_3, Z_4} \bigg\{\frac{d}{du}\bigg[ e^{-{u^2}/{2\|x_\alpha\|_2^2}} \\
	& \qquad \cdot\sigma'\bigg(\frac{v_{\beta, 1}u}{\|x_\alpha\|_2} + v_{\beta, -1}^\top Z_{-1}\bigg)\sigma'\bigg(\frac{v_{\gamma, 1}u}{\|x_\alpha\|_2} + v_{\gamma, -1}^\top Z_{-1}\bigg)\sigma'\bigg(\frac{v_{\xi, 1}u}{\|x_\alpha\|_2} + v_{\xi, -1}^\top Z_{-1}\bigg)\bigg]\bigg|_{u=0}\bigg\},
\end{align*}
where the last equality is by integration by part (which essentially defines $\delta(\cdot)$).
With some algebra, one can see that the derivative term in the RHS is equal to (with $u = 0$)
\begin{align*}
	& \frac{v_{\beta, 1}}{\|x_\alpha\|_2}\cdot \delta(v_{\beta, -1}^\top Z_{-1})  \sigma'(v_{\gamma, -1}^\top z_{-1}) \sigma'(v_{\xi, -1}^\top Z_{-1}) \\
	& \qquad + \frac{v_{\gamma, 1}}{\|x_\alpha\|_2}\cdot \sigma'(v_{\beta, -1}^\top Z_{-1})  \delta(v_{\gamma, -1}^\top z_{-1}) \sigma'(v_{\xi, -1}^\top Z_{-1})  \\
	& \qquad + \frac{v_{\xi, 1}}{\|x_\alpha\|_2}\cdot \sigma'(v_{\beta, -1}^\top Z_{-1})  \sigma'(v_{\gamma, -1}^\top z_{-1}) \delta(v_{\xi, -1}^\top Z_{-1}),
\end{align*}
For the first term in the above display, we have
\begin{align*}
	& \bbE \bigg[\delta(v_{\beta, -1}^\top Z_{-1})  \sigma'(v_{\gamma, -1}^\top z_{-1}) \sigma'(v_{\xi, -1}^\top Z_{-1})\bigg] \\
	& = \frac{1}{|v_{\beta, 2}|\sqrt{2\pi}} \bbE_{Z_3, Z_4}  \bigg[\sigma'(v_{\gamma, 3:4}^\top Z_{3:4}) \sigma'(v_{\xi, 3:4}^\top Z_{3:4})\bigg] \\
	& = \frac{\pi - \angle(v_{\gamma, 3:4}, v_{\xi, 3:4})}{2\pi\sqrt{2\pi} \cdot |v_{\beta, 2}|}.
\end{align*}
Meanwhile, we have
\begin{align*}
	& \bbE \bigg[\sigma'(v_{\beta, -1}^\top Z_{-1})  \delta(v_{\gamma, -1}^\top z_{-1}) \sigma'(v_{\xi, -1}^\top Z_{-1})\bigg] \\
	& = \bbE_{Z_4} \bigg[ \frac{1}{2\pi}\iint_{\bbR^2} e^{-(z_2^2+z_3^2)/2} \delta(v_{\gamma, 2}z_2 + v_{\gamma, 3}z_3) \sigma'(v_{\beta, 2}z_2)\sigma'(v_{\xi, 2}z_2 + v_{\xi, 3}z_3 + v_{\xi, 4} Z_4)\bigg].
\end{align*}
Using the following change of variables:
\[
	u = v_{\gamma, 2} z_2 + v_{\gamma, 3} z_3, \qquad w = z_2,
\]
so that
\[
	z_2 = w, \qquad z_3 = \frac{u - v_{\gamma, 2}w}{v_{\gamma, 3}},
\]
we have
\begin{align*}
	& \bbE \bigg[\sigma'(v_{\beta, -1}^\top Z_{-1})  \delta(v_{\gamma, -1}^\top z_{-1}) \sigma'(v_{\xi, -1}^\top Z_{-1})\bigg] \\
	& = \frac{1}{2\pi}\bbE_{Z_4} \bigg[ \iint_{\bbR^2} \frac{1}{|v_{\gamma, 3}|} \exp\{-\frac{-(w^2 + (v_{\gamma, 2} w/ v_{\gamma, 3})^2)}{2}\} \delta(u) \sigma'(v_{\beta, 2} w) \sigma'(v_{\xi, 2}w + v_{\xi, 3}\frac{u-v_{\gamma, 2}w}{v_{\gamma, 3}} + v_{\xi, 4}Z_4)\bigg] \\
	&  = \frac{1}{2\pi|v_{\gamma, 3}|} \bbE_{Z_4} \int_{-\infty}^{\infty} dw \exp\{-\frac{w^2}{2}\cdot (1+ (v_{\gamma, 2}/v_{\gamma, 3})^2)\} \sigma'(v_{\beta, 2} w) \sigma'\bigg((v_{\xi, 2} - \frac{v_{\xi, 3}v_{\gamma, 2}}{v_{\gamma, 3}}) w + v_{\xi, 4} Z_4\bigg) \\
	& = \frac{1}{\sqrt{2\pi}\cdot |v_{\gamma, 3}|}\bbE_{Z_4, W} \bigg[\sigma'(v_{\beta, 2} w)\sigma'\bigg((v_{\xi, 2} - \frac{v_{\xi, 3}v_{\gamma, 2}}{v_{\gamma, 3}}) w + v_{\xi, 4} Z_4\bigg)\bigg],
\end{align*}
where $W\sim\calN(0, \frac{1}{1 + v_{\gamma,2}^2/v_{\gamma, 3}^2})$. By rescaling, for $(X, Y)\sim\calN(0, I_2)$, we have
\begin{align*}
	& \bbE \bigg[\sigma'(v_{\beta, -1}^\top Z_{-1})  \delta(v_{\gamma, -1}^\top z_{-1}) \sigma'(v_{\xi, -1}^\top Z_{-1})\bigg] \\
	& = \frac{1}{\sqrt{2\pi} |v_{\gamma, 3}|} \bbE_{X, Y} \bigg[ \sigma'\bigg(\frac{v_{\beta, 2}}{\sqrt{1+ v_{\gamma, 2}^2 / v_{\gamma, 3}^2}}\cdot X\bigg)\sigma'\bigg( \frac{v_{\xi, 2}v_{\gamma, 3} - v_{\xi, 3}v_{\gamma, 2}}{v_{\gamma, 3}\sqrt{1+ v_{\gamma, 2}^2 / v_{\gamma, 3}^2}} X + v_{\xi, 4} Y \bigg)\bigg] \\
	& = \frac{1}{2\pi\sqrt{2\pi}|v_{\gamma, 3}|}  \cdot \bigg[\pi - \angle\bigg((\sqrt{1+ v_{\gamma, 2}^2 / v_{\gamma, 3}^2} , 0), (\frac{v_{\xi, 2}v_{\gamma, 3} - v_{\xi, 3}v_{\gamma, 2}}{v_{\gamma, 3}\sqrt{1+ v_{\gamma, 2}^2 / v_{\gamma, 3}^2}}, v_{\xi, 4})  \bigg)\bigg].
\end{align*}
Now let us consider
\begin{align*}
	& \bbE \bigg[\sigma'(v_{\beta, -1}^\top Z_{-1})  \sigma'(v_{\gamma, -1}^\top z_{-1}) \delta(v_{\xi, -1}^\top Z_{-1})\bigg]	 \\
	& = \frac{1}{2\pi\sqrt{2\pi}} \int_{\bbR^3} e^{-(z_2^2 + z_3^2 + z_4^2)/2} \delta(v_{\xi, 2}z_2 + v_{\xi, 3}z_3 + v_{\xi, 4} z_4) \sigma'(v_{\beta, 2}z_2) \sigma'(v_{\gamma, 2}z_2 + v_{\gamma, 3} z_3) dz_2d z_3 dz_4.
\end{align*}
We use the following change of variables:
\[
	x = v_{\xi, 2} z_2 + v_{\xi, 3} z_3 + v_{\xi, 4}, \ \ \ y = z_3, \ \ \ z = z_4,
\]
so that
\[
	z_2 = \frac{x - v_{\xi, 3} y - v_{\xi, 4} z}{v_{\xi, 2}},  \ \ \ z_3 = y, \ \ \ z_4 = z.
\]
Then we have
\begin{align*}
	& \bbE \bigg[\sigma'(v_{\beta, -1}^\top Z_{-1})  \sigma'(v_{\gamma, -1}^\top z_{-1}) \delta(v_{\xi, -1}^\top Z_{-1})\bigg] \\
	& = \frac{1}{2\pi\sqrt{2\pi} |v_{\xi, 2}|} \int_{\bbR^3} \exp\{-\frac{(\rvx  - v_{\xi, 3} y - v_{\xi, 4} z)^2/v_{\xi, 2}^2 + y^2 + z^2}{2}\} \cdot  \sigma'\bigg(\frac{v_{\beta, 2} (\rvx  - v_{\xi, 3} y - v_{\xi, 4}z)}{v_{\xi, 2}}\bigg) \\
	& \qquad \cdot \sigma'\bigg(\frac{v_{\gamma , 2} (\rvx  - v_{\xi, 3}y - v_{\xi, 4} z)}{v_{\xi, 2}}\bigg) \delta(\rvx ) dxdydz,
\end{align*}
where $1/|v_{\xi, 2}|$ is the Jacobian term when we do change of variables. Let us define $\Sigma_\xi$ via
\[
	\Sigma_\xi^{-1} =  
	\begin{pmatrix}
		1 + v_{\xi, 3}^2/v_{\xi, 2}^2 & v_{\xi, 3} v_{\xi, 4} / v_{\xi, 2} \\
		v_{\xi, 3} v_{\xi, 4} / v_{\xi, 2} & 1 + v_{\xi, 4}^2 / v_{\xi, 2}^2
	\end{pmatrix}.
\]
Integrating $x$ out, we get
\begin{align*}
	& \bbE \bigg[\sigma'(v_{\beta, -1}^\top Z_{-1})  \sigma'(v_{\gamma, -1}^\top z_{-1}) \delta(v_{\xi, -1}^\top Z_{-1})\bigg] \\
	& = \frac{\det(\Sigma_\xi)^{1/2}}{\sqrt{2\pi} |v_{\xi, 2}|} \frac{\det(\Sigma_\xi)^{-1/2}}{2\pi} \int_{\bbR^3} \exp\bigg\{-(y, z) \Sigma_\xi^{-1} (y, z)^\top/2\bigg\} \sigma'\bigg(-\frac{v_{\beta, 2} (v_{\xi, 3} y + v_{\xi, 4} z)}{v_{\xi, 2}}\bigg) \\
	& \qquad \cdot \sigma'\bigg((v_{\gamma, 3} - \frac{v_{\gamma, 2} v_{\xi, 3}}{v_{\xi, 2}} ) y - \frac{v_{\gamma, 2} v_{\xi, 4}}{v_{\xi, 2}}z\bigg) dydz \\
	& = \frac{\det(\Sigma_\xi)^{1/2}}{\sqrt{2\pi} |v_{\xi, 2}|} \bbE \bigg[\sigma'\bigg(-\frac{v_{\beta, 2} (v_{\xi, 3} Y + v_{\xi, 4} Z)}{v_{\xi, 2}}\bigg) \cdot \sigma'\bigg((v_{\gamma, 3} - \frac{v_{\gamma, 2} v_{\xi, 3}}{v_{\xi, 2}} ) Y - \frac{v_{\gamma, 2} v_{\xi, 4}}{v_{\xi, 2}}Z\bigg)\bigg],
\end{align*}
where $(Y, Z)\sim \calN(0, \Sigma_\xi)$. Note that for another independent vector $(\tilde Y, \tilde Z)\sim \calN(0, I_2)$, we have
\[
	(Y, Z) =_d \Sigma_\xi^{1/2} (\tilde Y, \tilde Z).
\]
Hence, we have
\begin{align*}
	& \bbE \bigg[\sigma'(v_{\beta, -1}^\top Z_{-1})  \sigma'(v_{\gamma, -1}^\top z_{-1}) \delta(v_{\xi, -1}^\top Z_{-1})\bigg] \\
	& = \frac{\det(\Sigma_\xi)^{1/2}}{\sqrt{2\pi} |v_{\xi, 2}|} \bbE \bigg[\sigma'\bigg(-\frac{v_{\beta, 2}}{v_{\xi, 2}} \cdot (v_{\xi, 3}, v_{\xi, 4}) \cdot \Sigma_\xi^{1/2}\cdot (\tilde Y, \tilde Z)^\top \bigg) \cdot \sigma'\bigg((v_{\gamma, 3} - \frac{v_{\gamma, 2} v_{\xi, 3}}{v_{\xi, 2}} ,   - \frac{v_{\gamma, 2} v_{\xi, 4}}{v_{\xi, 2}} ) \cdot \Sigma_\xi^{1/2} (\tilde Y, \tilde Z)^\top\bigg)\bigg] \\
	& = \frac{\det(\Sigma_\xi)^{1/2}}{2\pi\sqrt{2\pi} |v_{\xi, 2}|} \cdot \bigg[\pi - \angle\bigg( -\frac{v_{\beta, 2}}{v_{\xi, 2}} \cdot  \Sigma_\xi^{1/2} \cdot 
	\begin{pmatrix}
		v_{\xi, 3} \\
		v_{\xi, 4}
	\end{pmatrix}, 
	\Sigma_{\xi}^{1/2}\cdot
	\begin{pmatrix}
		v_{\gamma, 3} - \frac{v_{\gamma, 2} v_{\xi, 3}}{v_{\xi, 2}}\\
		- \frac{v_{\gamma, 2} v_{\xi, 4}}{v_{\xi, 2}}
	\end{pmatrix}
	\bigg)
	\bigg].
\end{align*}

In summary, we get
\begin{align*}
	\bbE \RN{5} & = \frac{-3m}{\|x_\alpha \|_2^2 \sqrt{2\pi}} \cdot \bigg\{ v_{\beta, 1} \cdot T_1 + v_{\gamma, 1} \cdot T_2 + v_{\xi, 1} T_3 \bigg\},
\end{align*}
where 
\begin{align*}
	T_1 & = \frac{\pi - \angle(v_{\gamma, 3:4}, v_{\xi, 3:4})}{2\pi\sqrt{2\pi} \cdot |v_{\beta, 2}|} \\
	T_2 & = \frac{1}{2\pi\sqrt{2\pi}|v_{\gamma, 3}|}  \cdot \bigg[\pi - \angle\bigg(
	\begin{pmatrix}
		\sqrt{1+ v_{\gamma, 2}^2 / v_{\gamma, 3}^2} \\
		0
	\end{pmatrix},
	\begin{pmatrix}
		\frac{v_{\xi, 2}v_{\gamma, 3} - v_{\xi, 3}v_{\gamma, 2}}{v_{\gamma, 3}\sqrt{1+ v_{\gamma, 2}^2 / v_{\gamma, 3}^2}} \\
		 v_{\xi, 4}
	\end{pmatrix} \bigg)\bigg] \\
	T_3 & = \frac{\det(\Sigma_\xi)^{1/2}}{2\pi\sqrt{2\pi} |v_{\xi, 2}|} \cdot \bigg[\pi - \angle\bigg( -\frac{v_{\beta, 2}}{v_{\xi, 2}} \cdot  \Sigma_\xi^{1/2} \cdot 
	\begin{pmatrix}
		v_{\xi, 3} \\
		v_{\xi, 4}
	\end{pmatrix}, 
	\Sigma_{\xi}^{1/2}\cdot
	\begin{pmatrix}
		v_{\gamma, 3} - \frac{v_{\gamma, 2} v_{\xi, 3}}{v_{\xi, 2}}\\
		- \frac{v_{\gamma, 2} v_{\xi, 4}}{v_{\xi, 2}}
	\end{pmatrix}
	\bigg)
	\bigg],
\end{align*}	
where
\[
	\Sigma_\xi =\begin{pmatrix}
		1 + v_{\xi, 3}^2/v_{\xi, 2}^2 & v_{\xi, 3} v_{\xi, 4} / v_{\xi, 2} \\
		v_{\xi, 3} v_{\xi, 4} / v_{\xi, 2} & 1 + v_{\xi, 4}^2 / v_{\xi, 2}^2
	\end{pmatrix}^{-1}. 
\]

\section{Connections and Differences to Previous Works on Label-Aware Kernels}\label{append:connections}
We discuss the relation between our proposed kernels and two lines of research on label-aware kernels, namely the Kernel Target Alignment (KTA) and the Information Bottleneck (IB) principle. 

{\bf Connections to KTA.} Recall that our higher-order regression based kernel is 
\begin{equation*}
	K^{(\textnormal{HR})}(\rvx, \rvx') :  = \bbE_\init K^{(2)}_0(\rvx, \rvx') + \lambda \gZ(\rvx, \rvx', S), 
\end{equation*}
where $\gZ(\rvx, \rvx', S)$ is an estimator of $(\textnormal{label of }\rvx) \times (\textnormal{label of }\rvx')$. From a high-level, this can be regarded as a specific way to to align with the ``optimal'' kernel $\rvy \rvy^\top$, because 
\[
    \la\rmK^{(\textnormal{HR})}, \rvy\rvy^\top\ra = \la\bbE_\init [\rmK^{(2)}_0] , \rvy\rvy^\top\ra + \lambda \la \boldsymbol{\calZ}, \rvy\rvy^\top\ra,
\]
and the term $\la \boldsymbol{\calZ}, \rvy\rvy^\top\ra$ is close to one as $\boldsymbol{\calZ}$ estimates $\rvy\rvy^\top$ by construction.

{\bf Relations to IB.} Consider the following the model fitting process: $Y\to X \to T \to \hat Y$. That is: 1) The nature generates a label $Y \in \{\pm 1\}$, e.g., a cat; 2) Given the label $Y$, the nature further generates a ``raw'' feature $X$, e.g., an image of a cat; 3) We try to find a feature map which maps $X$ to $T$; 4) We use $T$ to generate a prediction $\hat Y$. 

The IB principle gives a way to justify ``what kind of $T$ is optimal''. More explicitly, it poses the following optimization problem:
\[
	\min_{\substack{p_{T|X}\\ p_{Y|T}\\ p_T}} I(p_X; p_T) - \beta I(p_T; p_Y),
\]
where we let $p_{A|B}$ to be the conditional density of $A$ conditional on $B$. Then the ``optimal'' feature $T$ is a randomized map which sends a specific realization $X = \rvx$ to a random feature $T \sim p_{T|X=\rvx}$. 

Note that in the IB formulation, the optimal feature $T$ \emph{has no explicit dependence on $Y$} --- it only depends on $Y$ through $X$. This is in sharp contrast to our formulation, where we allow the feature to have explicit dependence on $Y$.

To further illustrate this point, it has been shown that any optimal solution to the IB optimization problem must satisfy the following set of \emph{self-consistent equations} \citep{tishby2000information}:
\begin{align*}
	p_{T|X}(t|x)& = \frac{p_T(t)}{Z(x; \beta)} \exp\bigg\{ -\beta D_{KL}(p_{Y|X} \| p_{Y|T})\bigg\}\\
	p(t) & = \int p_{T|X}(t | x) p_X(x) dx \\
	p_{Y|T}(y|t) & = \int p_{Y|X}(y|x) p_{X|T}(x|t) dx.
\end{align*}
Note that the distribution of the optimal feature $T$ only has dependence on $x$, and has no explicit dependence on $Y$, because $Y$ is marginalized in the KL divergence term.
\section{Experimental Details and Additional Results}
\label{sec:experimental-details}
\subsection{Details on Figure \ref{fig:le-cat-dog} and \ref{fig:le-chair-bench}}
\label{subsec:different-labeling-systems}
To experiment with different label systems, we use the MS COCO object detection dataset \citep{lin2014microsoft} because there can be various objects in the same image. In this experiment, we consider $50$ images with both cats and chairs, $50$ images with both cats and benches, $50$ images with both dogs and chairs, and $50$ images with both dogs and benches. An image with cats and chairs will include neither dogs nor benches, and similar for other cases. Therefore, in total, there are $100$ images with cats, $100$ images with dogs, $100$ images with chairs and $100$ images with benches. 

We then consider the same two-layer neural network as in Sec.~\ref{subsec:local-elasticity} and plot the relative ratio (Eq. \eqref{eq: relative_ratio}) for $\KK_t$ as $t$ increases, which gives Fig.~\ref{fig:le-cat-dog} and Fig.~\ref{fig:le-chair-bench}.



\subsection{The Architecture of CNN}
\label{subsec:CNN-architecture}
The architecture of the CNN used in Sec. \ref{subsec:kernel-regression} is as follows: there are seven convolutional layers with kernel size $3$ and padding number $1$, followed by a fully-connected layer. The output channels of each convolutional layer are: $16k$, $16k$, $16k$, $32k$, $32k$, $64k$, and $64k$, where by default $k=1$. But we increase $k$ to $16$, $8$, $4$ to get better CNN performance for multi-class classification with $2000$, $5000$, and $10000$ training examples in Table \ref{table:multiple-kernel-regression}. The strides for each convolutional layer is $1$ except the fourth and the last, which are $2$.

\subsection{Details on \lantk-HR}\label{subsec:details-of-Z}
{\bf The hyper-parameter $\lambda$.} Based on our experiments, the test accuracy is usually a concave function of $\lambda$. So we simply choose the best value among $\{0.001, 0.01, 0.1, 1.0\}$.

{\bf Choice of $\boldsymbol{\gZ(\rvx, \rvx', S)}$.} Our first choice of $\gZ$ is based on a kernel regression. Specifically, for \lantk-KR-V1, we consider 
\[
\gZ(\rvx, \rvx', S) = \sum_{i,j} \rvy_i\rvy_j \psi(\bbE_\init \rmK_0^{(2)}(\rvx, \rvx'), \bbE_\init \rmK_0^{(2)}(\rvx_i, \rvx_j)),
\]
and for \lantk-KR-V2, we consider 
\[
\gZ(\rvx, \rvx', S) = \sum_{i,j} ((\E_\init\rmKK_0)^{-1}\rvy)_i((\E_\init\rmKK_0)^{-1}\rvy)_j \psi(\bbE_\init \rmK_0^{(2)}(\rvx, \rvx'), \bbE_\init \rmK_0^{(2)}(\rvx_i, \rvx_j)),
\]
where 
\[
\psi(\phi_{ab}, \phi_{ij}) = \frac{B - (\phi_{ab} - \phi_{ij})^2}{n^2B - \sum_{s,t}(\phi_{ab} - \phi_{st})^2}
\] 
is a normalized similarity measure (smaller $(\phi_{ab} - \phi_{ij})^2$ indicates larger similarity) and $B$ is a constant which is set to be the largest $(\phi_{max} - \phi_{min})^2$ in the training data. Note that the change from $\rvy$ to $(\E_\init\rmKK_0)^{-1}\rvy$ in CNTK-V2 is inspired from the second term in $K^{(\textnormal{NTH})}(\rvx, \rvx')$. 

Our second choice of $\gZ$ is based on a linear regression with Fast-Johnson-Lindenstrauss-Transform (FJLT) \citep{ailon2009fast} and some hand-engineered features. FJLT is used to reduce the computational cost because there are $O(10^8)$ examples in the pairwise dataset\footnote{Our implementation is based on \url{https://github.com/michaelmathen/FJLT} and \url{https://github.com/dingluo/fwht}.}. And the hand-engineered features are as follows:
\begin{enumerate}
    \item For \lantk-FJLT-V1, we use the following features: $\bbE_\init\KK_0(\rvx, \rvx')$ (the label-agnostic kernel), $\cos\left<\rvx, \rvx'\right>$ (cos of the angle), $\rvx^T\rvx'$ (the inner product), $||\rvx||_2||\rvx'||_2$(the product of $L_2$ norms), $||\rvx - \rvx'||_2^2$ (Euclidean distance), $|\rvx - \rvx'|_1^1$($L_1$ distance), $\left<\rvx, \rvx'\right>$ (the angle between two vectors), $\sin\left<\rvx, \rvx'\right>$ (sin of the angle), $RBF(\rvx, \rvx')$ (RBF distance between two vectors), $\rho(\rvx, \rvx')$ (the Pearson correlation coefficient between two vectors). 
    \item For \lantk-FJLT-V2, in addition to the ten features used in \lantk-FJLT-V2, we also include the same $10$ features based on the top five principle components from principal component analysis (PCA).
\end{enumerate}



\subsection{Details on Training}
\label{subsec:experimental-settings}
We use the Adam optimizer \citep{kingma2014adam} with learning rates $3e^{-4}$ for $300$ epochs for the CNN in Sec.~\ref{subsec:kernel-regression} and for $500$ epochs for the 2-layer neural network in Sec.~\ref{subsec:local-elasticity}. For simplicity, we use the parameter with best test performance during the whole training trajectory. 

\end{document}